\documentclass[11pt]{article}
\usepackage[utf8]{inputenc}
\usepackage[letterpaper,margin=1.0in]{geometry}

\usepackage{microtype}
\usepackage{graphicx}
\usepackage{subfigure}
\usepackage{booktabs} 
\usepackage{hyperref}

\usepackage[numbers]{natbib}

\let\oldtop\top
\renewcommand{\top}{\!\oldtop}

\usepackage{amsmath,amsfonts,bm}
\usepackage{amssymb}

\usepackage{enumitem}

\def\eqref#1{equation~(\ref{#1})}

\def\floor#1{\left\lfloor #1 \right\rfloor}
\def\1{\bm{1}}

\newcommand{\norm}[1]{\left\| #1 \right\|_2}
\def\inner#1#2{\left\langle #1, #2 \right\rangle}

\def\eps{{\varepsilon}}

\def\vu{{\bm{u}}}
\def\vv{{\bm{v}}}
\def\vw{{\bm{w}}}
\def\vx{{\bm{x}}}
\def\vy{{\bm{y}}}
\def\vz{{\bm{z}}}

\def\mA{{\bf{A}}}

\def\gH{{\mathcal{H}}}

\def\gO{{\mathcal{O}}}

\def\gX{{\mathcal{X}}}
\def\gY{{\mathcal{Y}}}

\def\sR{{\mathbb{R}}}

\DeclareMathOperator*{\argmin}{arg\,min}

\DeclareMathOperator{\prox}{prox}

\usepackage{amsthm}
\theoremstyle{plain}
\newtheorem{thm}{Theorem}
\newtheorem{definition}{Definition}

\newtheorem{lemma}[thm]{Lemma}

\newtheorem*{remark}{Remark}
\newtheorem{corollary}[thm]{Corollary}

\usepackage{algorithm}
\usepackage{algorithmic}

\makeatletter
\def\Ddots{\mathinner{\mkern1mu\raise\p@
\vbox{\kern7\p@\hbox{.}}\mkern2mu
\raise4\p@\hbox{.}\mkern2mu\raise7\p@\hbox{.}\mkern1mu}}
\makeatother

\makeatletter
\newcommand*{\rom}[1]{\lowercase\expandafter\@slowromancap\romannumeral #1@}
\makeatother
\usepackage{makecell}
\usepackage{multirow}

\title{DIPPA: An improved Method for Bilinear Saddle Point Problems}

\author{
 Guangzeng Xie 
  \thanks{Academy for Advanced Interdisciplinary Studies,
  Peking University; email: \texttt{smsxgz@pku.edu.cn}.}
   \and
 Yuze Han 
  \thanks{School of Mathematical Sciences,
  Peking University; email: 
  \texttt{hanyuze97@pku.edu.cn}.}
  \and
 Zhihua Zhang 
  \thanks{School of Mathematical Sciences,
  Peking University; email:
  \texttt{zhzhang@math.pku.edu.cn}.}
}

\date{\today}
\begin{document}

\maketitle

\begin{abstract}
    This paper studies bilinear saddle point problems $\min_{\vx} \max_{\vy} g(\vx) + \vx^{\top} \mA \vy - h(\vy)$, where the functions $g, h$ are smooth and strongly-convex.
    When the gradient and proximal oracle related to $g$ and $h$ are accessible,
    optimal algorithms have already been developed in the literature \cite{chambolle2011first, palaniappan2016stochastic}.
    However, the proximal operator is not always easy to compute, especially in constraint zero-sum matrix games \cite{zhang2020sparsified}.
    This work proposes a new algorithm which only requires the access to the gradients of $g, h$.
    Our algorithm achieves a complexity upper bound $\tilde{\mathcal{O}}\left( \frac{\|\mA\|_2}{\sqrt{\mu_x \mu_y}} + \sqrt[4]{\kappa_x \kappa_y (\kappa_x + \kappa_y)}  \right)$ which has optimal dependency on the coupling condition number $\frac{\|\mA\|_2}{\sqrt{\mu_x \mu_y}}$ up to logarithmic factors.

\end{abstract}

\section{Introduction}
We consider the convex-concave bilinear saddle point problem of the following form
\begin{align}\label{prob:bilinear}
    \min_{\vx \in \gX} \max_{\vy \in \gY} f(\vx, \vy) \triangleq g(\vx) + \inner{\vx}{\mA \vy} - h(\vy).
\end{align}
This formulation arises in several popular machine learning applications such as matrix games~\cite{carmon2019variance, carmon2020coordinate, ibrahim2019linear}, regularized empirical risk minimization~\cite{zhang2017stochastic,tan2018stochastic}, AUC maximization~\cite{ying2016stochastic,shen2018towards}, prediction and regression problems~\cite{taskar2005structured,xu2004maximum}, reinforcement learning~\cite{du2017stochastic,dai2018sbeed}.

We study the most fundamental setting where $g$ is $L_x$-smooth and $\mu_x$-strongly convex, and $h$ is $L_y$-smooth and $\mu_y$-strongly convex.
For the first-order algorithms which iterate with gradient and proximal point operation of $g, h$, it has been shown that the upper complexity bounds \cite{chambolle2011first, palaniappan2016stochastic} of $\gO\left( \left( \frac{\norm{\mA}}{\sqrt{\mu_x \mu_y}} + 1 \right) \log\left(\frac{1}{\eps}\right) \right)$ match the lower complexity bound \cite{zhang2019lower} for solving bilinear saddle point problems. However, finding the exact proximal point of $g$ and $h$ could be very costly and impractical in some actual implementations.
One example is matrix games with extra cost functions $g$ and $h$.
In this case, the total cost function of the first player is $g(\vx) + \inner{\vx}{\mA \vy} $, and that of the second player is $h(\vy) - \inner{\vx}{\mA \vy}$. The forms of $g$ and $h$ could be complicated.
Moreover, \citet{kanzow2016augmented} proposed an augmented Lagrangian-type algorithm for solving generalized Nash equilibrium problems. Applying this algorithm to solve equality constrained matrix games \cite{zhang2020sparsified}, we obtain a subproblem of the form (\ref{prob:bilinear}). Even if the constraints are linear in $\vx$ and $\vy$, the complexity of calculating the exact proximal point of $g$ and $h$ could be unacceptable.

In this paper, we consider first-order algorithms for solving the bilinear saddle point problem (\ref{prob:bilinear}) with assuming that only the gradients of $g$ and $h$ are available.
In this setting, \citet{ibrahim2019linear, zhang2019lower} proved a gradient complexity lower bound $\Omega\left(\left(\frac{\norm{\mA}}{\sqrt{\mu_x \mu_y}} + \sqrt{\kappa_x + \kappa_y} \right) \log\left(\frac{1}{\eps}\right) \right)$.
On the other hand, we can modify the algorithm in \cite{chambolle2011first} to adapt to this setting with employing Accelerated Gradient Descent (AGD) to approximately solve each proximal point of $g, h$. This inexact version achieves an upper bound of $\tilde\gO\left(\left(\sqrt{\frac{\norm{\mA} L}{\mu_x \mu_y}} + \sqrt{\kappa_x + \kappa_y}\right)\log\left( \frac{1}{\eps} \right)\right)$ where $L = \max \left\{ \norm{\mA}, L_x, L_y \right\}$ (see Section \ref{sec:aipfb} in Appendix for more details).
And recently, \citet{wang2020improved} showed a same upper bound in a more general case.
This upper bound is tight when $\vx$ and $\vy$ are approximately decoupled (i.e., $\norm{\mA} < \max\{\mu_x, \mu_y\}$) or the coupling matrix $A$ is dominant in Problem (\ref{prob:bilinear}) (i.e., $\norm{\mA} > \max\{L_x, L_y\}$).
However, in the intermediate state, this upper bound would no longer be tight (see Figure \ref{fig:compare} for illustration).

In this work, we propose a new algorithm Double Inexact Proximal Point Algorithm (Algorithm \ref{algo:dippa}) and prove a convergence rate of
\begin{align*}
    \tilde\gO\left(\left(\frac{\norm{\mA}}{\sqrt{\mu_x \mu_y}} + \sqrt[4]{\kappa_x \kappa_y (\kappa_x + \kappa_y)}\right)\log\left( \frac{1}{\eps} \right)\right).
\end{align*}
Our upper bound enjoys a tight dependency on the coupling condition number $\frac{\norm{\mA}}{\sqrt{\mu_x \mu_y}}$ while suffers from an extra factor $\sqrt[4]{\frac{\kappa_x \kappa_y}{\kappa_x + \kappa_y}}$ on the condition numbers of $g$ and $h$. Our method is better than previous upper bounds for $\norm{\mA} = \Omega(\sqrt{L_x \mu_y + L_y \mu_x})$ and matches lower bound for $\norm{\mA} = \Omega( \sqrt[4]{L_x L_y (L_x \mu_y + L_y \mu_x)})$ (see Figure \ref{fig:compare} for illustration).
However, our method does not perform well in the case of weak coupling where $\norm{\mA} = \gO(\sqrt{L_x \mu_y + L_y \mu_x})$.

\begin{figure}[t]
    \vskip 0.2in
    \begin{center}
        \centerline{\includegraphics[width=0.5\columnwidth]{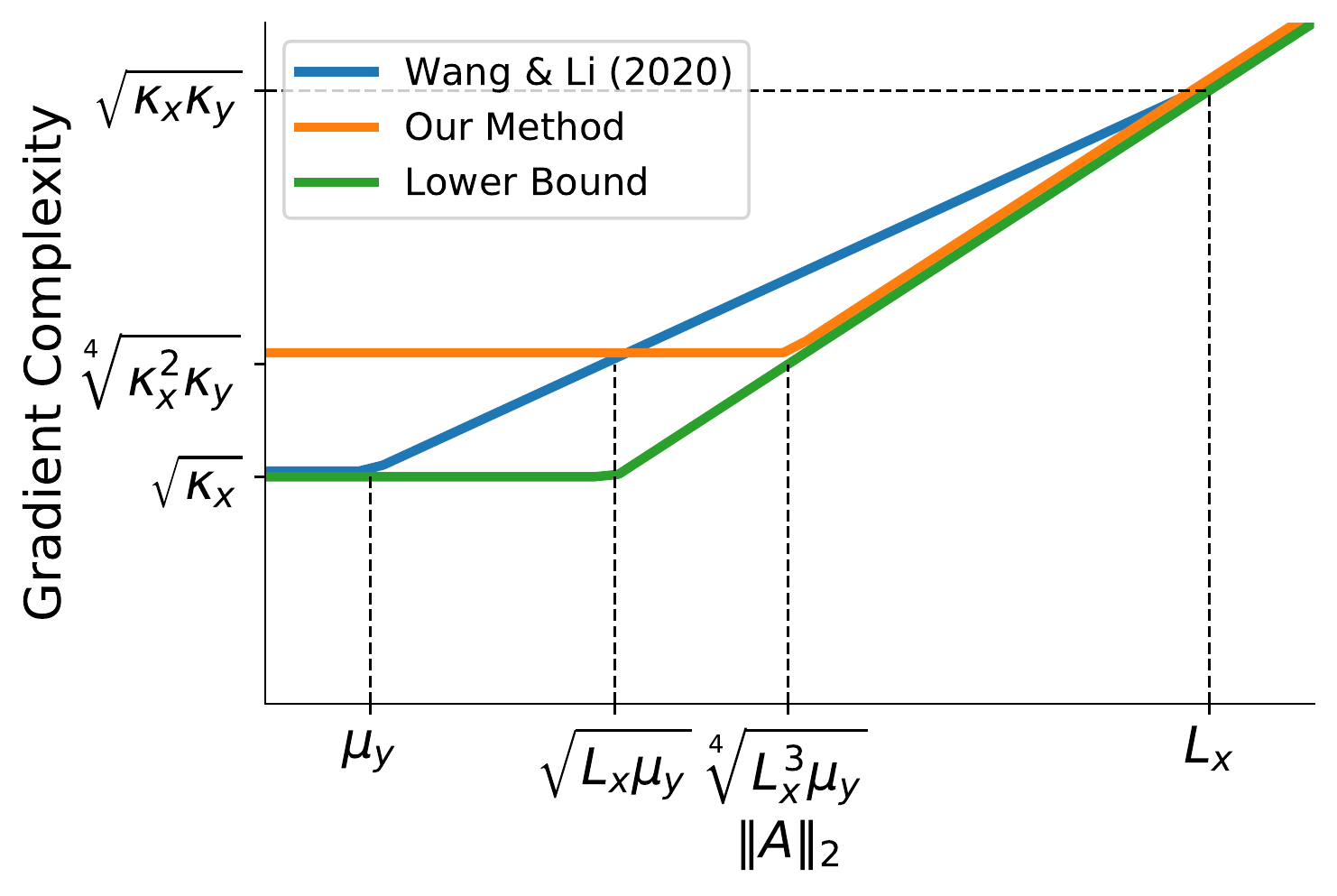}}
        \caption{Comparison of previous upper bounds \cite{wang2020improved}, lower bound \cite{zhang2019lower} and the results in this paper where $L_x = L_y, \mu_x < \mu_y$, ignoring logarithmic factors. We demonstrate the upper bounds and the lower bounds as a function of $\norm{A}$ with other fixed  parameters. }
        \label{fig:compare}
    \end{center}
    \vskip -0.2in
\end{figure}

The remainder of the paper is organized as follows.
We present preliminaries of the saddle point problems in Section \ref{sec:preliminary}.
Then we review some related work in Section \ref{sec:related} and two algorithm AGD and APFB which will be employed to solve subproblems of our method in Section \ref{sec:compnents}.
In Section \ref{sec:main}, we present the details of our method and provide a brief sketch of the analysis.
We conclude our work in Section \ref{sec:conclusion} and all the details of the proof can be found in Appendix.
Moreover, we also provide an inexact version of APFB in Section \ref{sec:aipfb} in Appendix.

\section{Preliminaries}\label{sec:preliminary}
In this paper, we use $\norm{A}$ to denote the spectral norm of $\mA$, i.e., the largest singular value of $\mA$.

Then we review some standard definitions of strong convexity and smoothness.
For a differentiable function $\varphi\colon \gX \to \sR$, $\varphi$ is said to be $\ell$-smooth if its gradient is $\ell$-Lipschitz continuous; that is, for any $\vx_1, \vx_2 \in \gX$, we have
\begin{align*}
    \|\nabla \varphi(\vx_1) - \nabla \varphi(\vx_2)\|_2 \leq \ell \norm{\vx_1 - \vx_2}.
\end{align*}
Moreover, $\varphi$ is said to be $\mu$-strongly convex, if for any $\vx_1, \vx_2 \in \gX$ we have
\begin{align*}
    \varphi(\vx_2) \ge \varphi(\vx_1) + \inner{\nabla \varphi(\vx_1)}{\vx_2 - \vx_1} + \frac{\mu}{2} \norm{\vx_2 - \vx_1}^2.
\end{align*}


The following lemma is useful in our analysis.
\begin{lemma}\label{lem:smooth}
    Let $\varphi$ be $\ell$-smooth and $\mu$-strongly convex on $\sR^d$. Then for all $\vx_1, \vx_2 \in \sR^d$, one has
    \begin{align*}
        \inner{\nabla \varphi(\vx_1) - \nabla \varphi(\vx_2)}{\vx_1 - \vx_2}
        \ge \frac{\ell \mu}{\ell + \mu} \norm{\vx_1 - \vx_2}^2 + \frac{1}{\ell + \mu} \norm{\nabla \varphi(\vx_1) - \nabla \varphi(\vx_2)}^2.
    \end{align*}
\end{lemma}

In this work, we are interested in the class of bilinear functions $f: \sR^{d_x} \times \sR^{d_y}$ of the form
\begin{align*}
    f(\vx, \vy) = g(\vx) + \inner{\vx}{\mA \vy} - h(y),
\end{align*}
where $g\colon \sR^{d_x} \to \sR$ is $L_x$-smooth and $\mu_x$-strongly convex, and $h\colon \sR^{d_y} \to \sR$ is $L_y$-smooth and $\mu_y$-strongly convex. And we denote $\kappa_x = L_x/\mu_x$ and $\kappa_y = L_y/\mu_y$.

Without loss of generality, we can assume that $L_x = L_y$. Otherwise, one can rescale the variables and take $\hat{f}(\vx, \vy) = f(\vx, \sqrt{L_x/L_y} \vy)$. It is not hard to check that this rescaling will not change condition numbers $\kappa_x, \kappa_y$ and coupling condition number $\frac{\norm{A}}{\sqrt{\mu_x \mu_y}}$.

Let the proximal operator related to $f$ at point $(\vx, \vy)$ be
\begin{align*}
    \prox_{f}(\vx, \vy) = \min_{\vu} \max_{\vv} f(\vu, \vv) + \frac{1}{2} \norm{\vu - \vx}^2 - \frac{1}{2} \norm{\vv - \vy}^2.
\end{align*}
If $f(\vx, \vy) = \inner{\vx}{\mA \vy}$, we use $\prox_{\mA}(\vx, \vy)$ for simplicity.

The optimal solution of the convex-concave minimax optimization problem $\min_{\vx} \max_{\vy} f(\vx, \vy)$ is the saddle point $(\vx^*, \vy^*)$ defined as follows.
\begin{definition}\label{def:saddle}
    $(\vx^*, \vy^*)$ is a saddle point of $f\colon \gX \times \gY \to \sR$ if for any $\vx \in \gX$ and $\vy \in \gY$, there holds
    \begin{align*}
        f(\vx^*,\vy) \leq f(\vx^*,\vy^*) \leq f(\vx,\vy^*).
    \end{align*}
\end{definition}

For strongly convex-strongly concave functions, it is well known that such a saddle point exists and is
unique. Given a tolerance $\eps > 0$, our goal is to find an $\eps$-saddle point which is defined as follows.
\begin{definition}
    $(\hat\vx, \hat\vy)$ is called an $\eps$-saddle point of $f$ if
    \begin{align*}
        \norm{\hat\vx - \vx^*} + \norm{\hat\vy - \vy^*} \le \eps.
    \end{align*}
\end{definition}

In this work, we focus on first-order algorithms which have access to oracle $\gH_{f}$.
For an inquiry on any point $(\vx, \vy)$, the oracle returns
\begin{align}\label{eq:oracle}
    \gH_{f}(\vx, \vy) \triangleq \{ \nabla g(\vx), \nabla h(\vy), \mA \vy, \mA^{\top} \vx \}.
\end{align}
Given an initial point $(\vx_0, \vy_0)$, at the $k$-th iteration,
a first-order algorithm calls the oracle on $(\vx_{k-1}, \vy_{k-1})$ and then obtains a new point $(\vx_{k}, \vy_{k})$. And $\vx_{k}$ and $\vy_{k}$ lie in two different vector spaces:
\begin{align*}
    \vx_{k} \in \mathrm{span} \{ & \vx_0, \nabla g(\vx_0), \dots, \nabla g(\vx_{k-1}), \mA \vy_0, \dots, \mA \vy_{k-1}  \},                             \\
    \vy_{k} \in \mathrm{span} \{ & \vy_0, \dots, \vy_k, \nabla h(\vy_0), \dots, \nabla h(\vy_{k-1}), \mA^{\top} \vx_0, \dots, \mA^{\top} \vx_{k-1}  \}.
\end{align*}

\section{Related Work}\label{sec:related}
There are many algorithms designed for the convex-concave saddle point problems, including extragradient (EG) algorithm \cite{korpelevich1976extragradient, tseng1995linear, mokhtari2019proximal, mokhtari2019unified},
reflected gradient descent ascent \cite{chambolle2011first, malitsky2015projected, yadav2017stabilizing},
optimistic gradient descent ascent (OGDA) \cite{daskalakis2018training, mokhtari2019proximal, mokhtari2019unified},
and other variants \cite{rakhlin2013online, rakhlin2013optimization,mertikopoulos2019optimistic}.

The bilinear case has also been studied extensively \cite{nesterov2005smooth, chambolle2011first, he2016accelerated}. And, \citet{kolossoski2017accelerated} introduced convergence results even when the feasible space is non-Euclidean and \citet{chen2014optimal, chen2017accelerated} proposed optimal algorithms for solving a special class of stochastic saddle point problems.
For a class of matrix games where $g, h$ are zero functions, \citet{carmon2019variance, carmon2020coordinate} developed variance reduction and Coordinate methods to solve the games.

For the strongly convex-strongly concave minimax problems, \citet{tseng1995linear, nesterov2006solving} provided upper bounds based on a variational inequality. Moreover, \citet{gidel2018variational, mokhtari2019unified} derived upper bounds for the OGDA algorithm. Recently, \citet{lin2020near, wang2020improved} proposed algorithms based on the approximately proximal technique and improved the upper bounds. On the other hand, \citet{ibrahim2019linear, zhang2019lower} established a lower complexity bound among all the first-order algorithms.
We provide a comparison between our results and existing results in the literature in Table \ref{table-comparison}.

\renewcommand{\arraystretch}{2}

\begin{table}[t]
    \caption{Comparison of gradient complexities to find an $\eps$-saddle point of Problem (\ref{prob:bilinear}) where $L_x = L_y$, $\kappa_x = L_x / \mu_x, \kappa_y = L_y / \mu_y$ and $L = \max\left\{\norm{A}, L_x\right\}$. The notations $\tilde\gO$ and $\tilde\Omega$ have ignored some logarithmic factors. }
    \label{table-comparison}
    \vskip 0.15in
    \setlength{\tabcolsep}{3pt}
    \begin{center}
        \begin{tabular}{|c|c|}
            \hline
            References                           & Gradient Complexity                                                                                               \\ [0.1cm]
            \hline
            \makecell[c]{\citet{nesterov2006solving}                                                                                                                 \\ \citet{mokhtari2019unified}} & $ \tilde\gO\left(\frac{\norm{\mA}}{\mu_x} + \frac{\norm{\mA}}{\mu_y}   + \kappa_x + \kappa_y\right) $ \\[0.1cm]
            \hline
            \citet{lin2020near}                  & $ \tilde\gO\left(\frac{\norm{\mA}}{\sqrt{\mu_x \mu_y}} + \sqrt{\kappa_x \kappa_y}\right) $                        \\[0.1cm]
            \hline
            \citet{wang2020improved}             & \multirow{2}{*}{$\tilde\gO\left(\sqrt{\frac{\norm{\mA} L}{\mu_x \mu_y}} + \sqrt{\kappa_x + \kappa_y}\right)$}     \\[0.1cm]
            \cline{1-1}
            \makecell[c]{\citet{chambolle2011first}                                                                                                                  \\ Inexact version (Theorem \ref{thm:catalyst-aipfb}) }  &  \\[0.1cm]
            \hline
            This paper (Theorem \ref{thm:dippa}) & $\tilde\gO\left(\frac{\norm{\mA}}{\sqrt{\mu_x \mu_y}} + \sqrt[4]{\kappa_x \kappa_y (\kappa_x + \kappa_y)}\right)$ \\[0.1cm]
            \hline
            \makecell[c]{Lower bound                                                                                                                                 \\ \citet{ibrahim2019linear} \\ \citet{zhang2019lower} } & $\tilde\Omega\left(\frac{\norm{\mA}}{\sqrt{\mu_x \mu_y}} + \sqrt{\kappa_x + \kappa_y}\right)$  \\
            \hline
        \end{tabular}
    \end{center}
\end{table}
\renewcommand{\arraystretch}{1}

\section{Algorithm Components}\label{sec:compnents}
In this section, we present two main algorithm components. Both of them are crucial for our algorithms.

\subsection{Nesterov’s Accelerated Gradient Descent}
We present a version of Nesterov's Accelerated Gradient Descent (AGD) in Algorithm \ref{algo:agd} which is widely-used to
minimize an $\ell$-smooth and $\mu$-strongly convex function $f$ \cite{nesterov2018lectures}. Moreover, AGD is shown to be optimal among all the first-order algorithms for smooth and strongly convex optimization.
\begin{algorithm}[ht]
    \caption{AGD}\label{algo:agd}
    \begin{algorithmic}[1]
        \STATE \textbf{Input:} function $\varphi$, initial point $\vx_0$, smoothness $\ell$, strongly convex module $\mu$, run-time $T$. \\[0.1cm]
        \STATE \textbf{Initialize:} $\tilde\vx_0 = \vx_0$, $\eta = 1/\ell$, $\kappa = \ell / \mu$ and $\theta = \frac{\sqrt{\kappa} - 1}{\sqrt{\kappa} + 1}$. \\[0.1cm]
        \STATE \textbf{for} $k = 1, \dots, T$ \textbf{do}\\[0.1cm]
        \STATE\quad $\vx_k = \tilde\vx_{k-1} - \eta \nabla \varphi(\tilde\vx_{k-1})$.\\[0.1cm]
        \STATE\quad $\tilde\vx_k = \vx_k + \theta (\vx_k - \vx_{k-1})$. \\[0.1cm]
        \STATE\textbf{end for} \\[0.1cm]
        \STATE \textbf{Output:} $\vx_T$.
    \end{algorithmic}
\end{algorithm}

The following theorem provides the convergence rate of AGD.
\begin{thm}\label{thm:agd}
    Assume that $\varphi$ is $\ell$-smooth and $\mu$-strongly convex. Then the output of Algorithm \ref{algo:agd} satisfies
    \begin{align*}
        \varphi(\vx_T) - \varphi(\vx^*) \le \frac{\ell + \mu}{2} \norm{\vx_0 - \vx^*}^2 \exp\left(-\frac{T}{\sqrt{\kappa}}\right),
    \end{align*}
    where $\kappa = \ell/\mu$ is the condition number, and $\vx^*$ is the unique global minimum of $\varphi$.
\end{thm}
A standard analysis of Theorem \ref{thm:agd} based on estimating sequence can be found in \cite{nesterov2018lectures}.
AGD will be used as a basic component for acceleration in this paper.

\subsection{Accelerated Proximal Forward-Backward Algorithm}
The Accelerated Proximal Forward-Backward Algorithm (APFB, Algorithm \ref{algo:apfb}) is proposed by \citet{chambolle2011first} which is an optimal method when the proximal oracle of $g, h$ is available \cite{zhang2019lower}. APFB takes the alternating order of updating $\vx$ and $\vy$ and employs momentum steps as well as AGD which yields acceleration.

\begin{algorithm}[ht]
    \caption{APFB}\label{algo:apfb}
    \begin{algorithmic}[1]
        \STATE \textbf{Input:} function $g, h$, coupling matrix $\mA$, initial point $\vx_0, \vy_0$, strongly convex module $\mu_x, \mu_y$, run-time $T$. \\[0.1cm]
        \STATE \textbf{Initialize:} $\tilde\vx_0 = \vx_0$, $\gamma = \frac{1}{\norm{\mA}}\sqrt{\frac{\mu_y}{\mu_x}}$, $\sigma = \frac{1}{\norm{\mA}}\sqrt{\frac{\mu_x}{\mu_y}}$ and $\theta = \frac{\norm{\mA}}{\sqrt{\mu_x \mu_y} + \norm{\mA}}$. \\[0.1cm]
        \STATE \textbf{for} $k = 1, \cdots, T$ \textbf{do}\\[0.1cm]
        \STATE\quad $\vy_k = \argmin_{\vv} h(\vv) + \frac{1}{2 \sigma} \norm{\vv - \vy_{k-1} - \sigma \mA^{\top} \tilde\vx_{k-1}}^2$. \\[0.1cm]
        \STATE\quad $\vx_k = \argmin_{\vu} g(\vu) + \frac{1}{2\gamma} \norm{\vu - \vx_{k-1} + \gamma \mA \vy_k}^2$. \\[0.1cm]
        \STATE\quad $\tilde\vx_k = \vx_k + \theta (\vx_k - \vx_{k-1})$. \\[0.1cm]
        \STATE\textbf{end for} \\[0.1cm]
        \STATE \textbf{Output:} $\vx_T, \vy_T$.
    \end{algorithmic}
\end{algorithm}

A theoretical guarantee for the APFB algorithm is presented in the following theorem. The proof of Theorem \ref{thm:apfb} can be found in Appendix Section B.
\begin{thm}\label{thm:apfb}
    Assume that $g$ is $\mu_x$-strongly convex and $h$ is $\mu_y$-strongly convex. Then the output of Algorithm \ref{algo:apfb} satisfies
    \begin{small}\begin{align*}
            \mu_x \norm{\vx_T - \vx^*}^2 + \mu_y \norm{\vy_T - \vy^*}^2 \le \theta^{T-1}\frac{ \norm{\mA}}{\sqrt{\mu_x \mu_y}}
            \left( \mu_x \norm{\vx_0 - \vx^*}^2 + \mu_y \norm{\vy_0 - \vy^*} \right),
        \end{align*}\end{small}
    where $\vx^*, \vy^*$ is the unique saddle point of $f(\vx, \vy) = g(\vx) + \inner{\vx}{\mA \vy} - h(\vy)$ and $\theta = \frac{\norm{\mA}}{\sqrt{\mu_x \mu_y} + \norm{\mA}}$.
\end{thm}

\section{Methodology}\label{sec:main}
In this section, we first consider balanced cases where $\kappa_x = \kappa_y$.
We introduce a Double Proximal Point Algorithm (DPPA, Algorithm \ref{algo:dppa}) with assuming that each subproblem can be solved exactly.
Then we present an inexact version of DPPA as Double Inexact Proximal Point Algorithm (DIPPA, Algorithm \ref{algo:dippa}) with solving the subproblems iteratively and show its theoretical guarantee for solving balanced bilinear saddle point problems (\ref{prob:bilinear}).
At last, we apply Catalyst framework with DIPPA to solve unbalanced bilinear saddle point problems.


\subsection{Double Proximal Point Algorithm for Balanced Cases}
We first consider balanced cases where $\mu_x = \mu_y$.
This implies $\kappa_x = \kappa_y$.
Our method is inspired from the algorithm Hermitian and skew-Hermitian splitting (HSS) \cite{Bai2003hermitian} which is designed for the non-Hermitian positive definite system of linear equations.
We present the DPPA algorithm for balanced cases in Algorithm \ref{algo:dppa}.
DPPA split the function $f$ into two parts $f = f_1 + f_2$ where $f_1(\vx, \vy) = g(\vx) - h(\vy)$ and $f_2(\vx, \vy) = \inner{\vx}{\mA \vy}$.
Moreover there are two proximal steps at each iteration of DPPA:
Line 6 performs a proximal step with respect to the function $f_1$, while Line 7 performs another proximal step related to the function $f_2$.

\begin{algorithm}[ht]
    \caption{DPPA for balanced cases}\label{algo:dppa}
    \begin{algorithmic}[1]
        \STATE \textbf{Input:} function $g, h$, coupling matrix $\mA$, initial point $(\vx_0, \vy_0)$, smoothness $L$, strongly convex module $\mu$, run-time $K$. \\[0.1cm]
        \STATE \textbf{Initialize:} $\alpha = 1/\sqrt{L \mu}$. \\[0.1cm]
        \STATE \textbf{for} $k = 1, \cdots, K$ \textbf{do}\\[0.1cm]
        \STATE\quad $\vz_{k} = \vx_{k-1} - \alpha \mA \vy_{k-1}$.\\[0.1cm]
        \STATE\quad $\vw_{k} = \vy_{k-1} + \alpha \mA^{\top} \vx_{k-1}$.\\[0.1cm]
        \STATE\quad $(\tilde\vx_{k}, \tilde\vy_{k}) = \prox_{\alpha (g - h)}(\vz_k, \vw_k)$. \\[0.1cm]
        \STATE\quad $(\vx_k, \vy_k) = \prox_{\alpha \mA} (2\tilde\vx_k - \vz_k, 2\tilde\vy_k - \vw_k)$. \\[0.1cm]
        \STATE\textbf{end for} \\[0.1cm]
        \STATE \textbf{Output:} $\vx_K, \vy_K$.
    \end{algorithmic}
\end{algorithm}

The theoretical guarantee for the algorithm DPPA in balanced cases is given in the following theorem.
\begin{thm}\label{thm:dppa}
    Assume that $g, h$ are both $L$-smooth and $\mu$-strongly convex. Denote $(\vx^*, \vy^*)$ to be the saddle point of the function $f(\vx, \vy) = g(\vx) + \vx^{\top} \mA \vy - h(\vy)$. Then the sequence $\{\vz_k, \vw_k\}_{k\ge 1}$ in Algorithm \ref{algo:dppa} satisfies
    \begin{align*}
        \norm{\vz_{k+1} - \vz^*}^2 + \norm{\vw_{k+1} - \vw^*}^2 \le \eta \left( \norm{\vz_k - \vz^*}^2 + \norm{\vw_k - \vw^*}^2 \right)
    \end{align*}
    where $\eta = \left(\frac{\sqrt{\kappa} - 1}{\sqrt{\kappa} + 1}\right)^2$, $\kappa = L/\mu$
    and $\vz^* = \vx^* - \alpha \mA \vy^*, \vw^* = \vy^* + \alpha \mA^{\top} \vx^*$.
\end{thm}

Using the Theorem \ref{thm:dppa}, we directly obtain that $\norm{\vx_k - \vx^*}^2 + \norm{\vy_k - \vy^*}$ converges to 0 linearly. We present our result in Corollary \ref{coro:dppa}.
\begin{corollary}\label{coro:dppa}
    Based on the same notations and assumptions of Theorem \ref{thm:dppa}, the output of Algorithm \ref{algo:dppa} satisfies
    \begin{align*}
        \norm{\vx_K - \vx^*}^2 + \norm{\vy_K - \vy^*}
        \le \frac{\norm{\mA}^2 + L \mu}{L \mu} \left( \norm{\vx_0 - \vx^*}^2 + \norm{\vy_0 - \vy^*} \right) \exp\left(-\frac{2K}{\sqrt{\kappa}}\right).
    \end{align*}
\end{corollary}
\begin{proof}
    Firstly, with $\alpha = \frac{1}{\sqrt{L \mu}}$, we note that
    \begin{align*}
        \norm{\vz_k - \vz^*}^2 + \norm{\vw_k - \vw^*}^2
         & = \norm{\vx_{k-1} - \vx^*}^2 + \norm{\vy_{k-1} - \vy^*}^2
        + \alpha^2 \left(\norm{\mA^{\top}(\vx_{k-1} - \vx^*)}^2 + \norm{\mA (\vy_{k-1} - \vy^*)}^2\right)                           \\
         & \le \left( 1 + \frac{\norm{\mA}^2}{L\mu} \right) \left( \norm{\vx_{k-1} - \vx^*}^2 + \norm{\vy_{k-1} - \vy^*}^2 \right).
    \end{align*}

    Therefore, we can conclude that
    \begin{align*}
        \norm{\vx_{T} - \vx^*}^2 + \norm{\vy_T - \vy^*}^2
         & \le \norm{\vz_{T+1} - \vz^*}^2 + \norm{\vw_{T+1} - \vw^*}^2                                                                                           \\
         & \le \eta^{T} \left( \norm{\vz_{1} - \vz^*}^2 + \norm{\vw_{1} - \vw^*}^2 \right)                                                                       \\
         & \le \exp\left(-\frac{2T}{\sqrt{\kappa}}\right) \frac{\norm{\mA}^2 + L \mu}{L \mu} \left( \norm{\vx_{0} - \vx^*}^2 + \norm{\vy_{0} - \vy^*}^2 \right),
    \end{align*}
    where we have used that $\frac{\sqrt{\kappa} - 1}{\sqrt{\kappa} + 1} \le 1 - \frac{1}{\sqrt{\kappa}} \le \exp\left(-\frac{1}{\sqrt{\kappa}}\right)$.
\end{proof}

\subsection{Double Inexact Proximal Point Algorithm for Balanced Cases}
We provide the details of Double Inexact Proximal Point Algorithm (DIPPA) for balanced cases in Algorithm \ref{algo:dippa}.
We iteratively solve the two proximal steps in DPPA.
More specifically, we may employ AGD to approximately find the proximal point of the function $f_1(\vx, \vy) = g(\vx) - h(\vy)$ where the variables $\vx$ and $\vy$ are completely decoupled.
The second proximal point subproblem (\ref{prob:sub2}) can be solved by APFB since the proximal operator of $\tilde{g}_k$ and $\tilde{h}_k$ is easy to obtain. We note that subproblem (\ref{prob:sub2}) is quadratic, and finding the saddle point of $\tilde{f}_k$ is equivalent to solving the system of linear equations with coefficient matrix
$
    \begin{bmatrix}
        I                  & \alpha \mA \\
        -\alpha \mA^{\top} & I
    \end{bmatrix}
$, which also can be solved by some Krylov subspace methods \cite{Greenbaum1997iterative, Concus2007generalized}.
\begin{algorithm}[tb]
    \caption{DIPPA for balanced cases}\label{algo:dippa}
    \begin{algorithmic}[1]
        \STATE \textbf{Input:} function $g, h$, coupling matrix $\mA$, initial point $(\vx_0, \vy_0)$, smoothness $L$, strongly convex module $\mu$, run-time $K$, tolerance sequences $\{\eps_k\}_{k\ge 1}$ and $\{\delta_k\}_{k \ge 1}$. \\[0.1cm]
        \STATE \textbf{Initialize:} $\alpha = 1/\sqrt{L \mu}$. \\[0.1cm]
        \STATE \textbf{for} $k = 1, \cdots, K$ \textbf{do}\\[0.1cm]
        \STATE\quad $\vz_{k} = \vx_{k-1} - \alpha \mA \vy_{k-1}$.\\[0.1cm]
        \STATE\quad $\vw_{k} = \vy_{k-1} + \alpha \mA^{\top} \vx_{k-1}$.\\[0.1cm]
        \STATE\quad Let $G_k(\vx) = g(\vx) + \frac{1}{2\alpha} \norm{\vx - \vz_k }^2$.\\[0.15cm]
        \STATE\quad Find $\tilde\vx_{k}$ such that $G_k(\tilde\vx_{k}) - \min_{\vx} G_k(\vx) \le \eps_k.$ \\[0.15cm]
        \STATE\quad Let $H_k(\vy) = h(\vy) + \frac{1}{2\alpha} \norm{\vy - \vw_k}^2$.\\[0.15cm]
        \STATE\quad Find $\tilde\vy_{k}$ such that $H_k(\tilde\vy_{k}) - \min_{\vy} H_k(\vy) \le \eps_k.$ \\[0.15cm]
        \STATE\quad Obtain $(\vx_{k}, \vy_{k})$ to be $\delta_k$-saddle point of the following problem
        \begin{small}
            \begin{equation}\label{prob:sub2}
                \begin{split}
                    \min_{\vx} \max_{\vy} &~ \tilde{f}_k(\vx, \vy) = \tilde{g}_k(\vx) + \inner{\vx}{\mA \vy} - \tilde{h}_k(\vy) \\
                    &\triangleq\frac{1}{2\alpha} \norm{\vx - 2\tilde\vx_{k} + \vz_k}^2 + \inner{\vx}{\mA \vy}
                    - \frac{1}{2\alpha} \norm{\vy - 2\tilde\vy_{k} + \vw_k}^2.
                \end{split}
            \end{equation}
        \end{small}
        \STATE\textbf{end for} \\[0.1cm]
        \STATE \textbf{Output:} $\vx_K, \vy_K$.
    \end{algorithmic}
\end{algorithm}

The convergence rate of the algorithm DIPPA in balanced cases is provided in the following theorem.
\begin{thm}\label{thm:dippa-out}
    Assume that $g, h$ are both $L$-smooth and $\mu$-strongly convex.
    Denote $(\vx^*, \vy^*)$ is the saddle point of the function $f = g(\vx) + \vx^{\top} \mA \vy - h(\vy)$.
    Set
    \begin{align*}
        \rho = \frac{1}{2\sqrt{\kappa}}, ~~~
         & C_0 = \norm{\vx_{0} - \vx^*}^2 + \norm{\vy_{0} - \vy^*}^2,                              \\
        \eps_k = \frac{C_0 \mu }{16} (1 - \rho)^{k+1}, ~~
         & \delta_k = \frac{C_0 L \mu }{2(1 + \sqrt{\kappa})(L\mu + \norm{A}^2)} (1 - \rho)^{k+1}.
    \end{align*}
    Then the output of Algorithm \ref{algo:dippa} satisfies
    \begin{align*}
        \norm{\vx_K - \vx^*}^2 + \norm{\vy_K - \vy^*} \le C \left( 1 - \rho \right)^{K} \left(\norm{\vx_{0} - \vx^*}^2 + \norm{\vy_{0} - \vy^*}^2\right),
    \end{align*}
    where
    $
        C = 4\sqrt{\kappa} + 1 + \frac{\norm{A}^2}{L \mu}
    $
    and $\kappa = L/\mu$.
\end{thm}

Then we upper bound the complexity of solving subproblems to analyze the total complexity of DIPPA.
\begin{lemma}\label{lem:dippa:inner}
    Consider the same assumption and the same definitions of $\rho$, $\eps_k$, $\delta_k$ and $C$ in Theorem \ref{thm:dippa-out}.
    In order to find $\eps_k$-optimal points $\tilde\vx_k$ ($\tilde\vy_k$) of $G_k$ ($H_k$), we need to run $\mathrm{AGD}$ $K_1$ steps, where
    \begin{align*}
        K_1 = \floor{\sqrt[4]{\kappa} \log\left(\frac{32 C (\sqrt{L} + \sqrt{\mu})^2 }{\mu(1 - \rho)}\right)} + 1.
    \end{align*}
    And in order to obtain $\delta_k$-saddle point $(\vx_k, \vy_k)$ of $\tilde{f}_k$, we need to run $\mathrm{APFB}$ $K_2$ steps, where
    \begin{align*}
        K_2 = \floor{ \left( \frac{\norm{\mA}}{\sqrt{L\mu}} + 1 \right) \log\left(\frac{20 C (1 + \sqrt{\kappa}) (L \mu + \norm{A}^2) }{L\mu (1 - \rho)}\right) } + 2.
    \end{align*}
\end{lemma}

Now we can provide the upper bound of total complexity of Algorithm \ref{algo:dippa} for solving balanced bilinear saddle point problems.
\begin{thm}\label{thm:dippa:0}
    The total queries to Oracle (\ref{eq:oracle}) needed by Algorithm \ref{algo:dippa} to produce $\eps$-saddle point of $f$ is at most
    \begin{align*}
        \tilde\gO\left( \left( \frac{\norm{\mA}}{\mu} + \kappa^{3/4} \right) \log\left(\frac{\norm{\vx_{0} - \vx^*}^2 + \norm{\vy_{0} - \vy^*}^2}{\eps}\right) \right),
    \end{align*}
    where the notation $\tilde\gO$ have omitted some logarithmic factors depending on $\kappa$ and $\frac{\norm{\mA}^2}{L \mu}$.
\end{thm}
\begin{proof}
    By Theorem \ref{thm:dippa-out}, in order to produce $\eps$-saddle point of $f$, we only need to run DIPPA $K$ steps, where
    \begin{align*}
        K = \floor{2 \sqrt{\kappa} \log \left( \frac{C (\norm{\vx_{0} - \vx^*}^2 + \norm{\vy_{0} - \vy^*}^2)}{\eps} \right)} + 1.
    \end{align*}
    Together with Lemma \ref{lem:dippa:inner}, the total complexity is upper bounded by
    \begin{align*}
        K(2 K_1 + K_2) = \tilde\gO \left( \sqrt{\kappa} \left( \frac{\norm{\mA}}{\sqrt{L\mu}} + \sqrt[4]{\kappa} \right) \log\left(\frac{\norm{\vx_{0} - \vx^*}^2 + \norm{\vy_{0} - \vy^*}^2}{\eps}\right) \right).
    \end{align*}
\end{proof}

\subsection{Catalyst-DIPPA for Unbalanced Cases}
Catalyst \cite{lin2018catalyst, yang2020catalyst} is a successful framework to accelerate existing first-order algorithms.
We present the details of Catalyst-DIPPA in Algorithm \ref{algo:catalyst-dippa}.
The idea is to repeatedly solve the following auxiliary balanced saddle point problems using DIPPA:
\begin{align*}
    \min_{\vx} \max_{\vy} f_k(\vx, \vy) \triangleq f(\vx, \vy) + \frac{\beta}{2} \norm{\vx - \tilde\vx_k}^2,
\end{align*}
where $\beta = \frac{L_x(\mu_y - \mu_x)}{L_x - \mu_y}$.
We remark that the function $f_k$ is balanced: the condition number corresponding to $\vy$ is   $\kappa_y$ and the condition number related to $\vx$ is
\begin{align*}
    \frac{L_x + \beta}{\mu_x + \beta} = \frac{L_x(L_x - \mu_y) + L_x (\mu_y - \mu_x)}{\mu_x(L_x - \mu_y) + L_x (\mu_y - \mu_x)} = \kappa_y,
\end{align*}
where we have recalled that $L_x = L_y$.
With the rescaling technique, we can apply DIPPA to solve the following saddle point problem
\begin{align*}
    \min_{\vx}\max_{\vy} \hat{f}_k(\vx, \vy) \triangleq f_k\left(\sqrt{\frac{L_x}{L_x + \beta}}\vx, \vy\right).
\end{align*}
Note that the coupling matrix of $\hat{f}_k$ is $\sqrt{\frac{L_x}{L_x + \beta}} \mA$. So the total gradient complexity of Catalyst-DIPPA is
\begin{align*}
     & \quad \tilde\gO\left( \sqrt{\frac{\mu_x + \beta}{\mu_x}} \right) \tilde\gO\left( \sqrt{\frac{L_x}{L_x + \beta}} \frac{\norm{\mA}}{\mu_y} + \kappa_y^{3/4} \right) \\
     & = \tilde\gO\left(\frac{\norm{\mA}}{\sqrt{\mu_x \mu_y}} + \sqrt{\kappa_x}\sqrt[4]{\kappa_y} \right).
\end{align*}

\begin{algorithm}[tb]
    \caption{Catalyst-DIPPA for unbalanced cases}\label{algo:catalyst-dippa}
    \begin{algorithmic}[1]
        \STATE \textbf{Input:} function $f$, initial point $(\vx_0, \vy_0)$, smoothness $L_x = L_y$, strongly convex module $\mu_x < \mu_y$, run-time $K$, accuracy sequence $\{\eps_k\}_{k\ge 1}$. \\[0.1cm]
        \STATE \textbf{Initialize:} $\beta = \frac{L_x(\mu_y - \mu_x)}{L_x - \mu_y}$, $q = \frac{\mu_x}{\mu_x + \beta}$, $\theta = \frac{1 - \sqrt{q}}{1 + \sqrt{q}}$ and $\tilde\vx_0 = \vx_0$. \\[0.1cm]
        \STATE \textbf{for} $k = 1, \cdots, K$ \textbf{do}\\[0.1cm]
        \STATE\quad Let $f_k(\vx, \vy) = f(\vx, \vy) + \frac{\beta}{2} \norm{\vx - \tilde\vx_k}^2$.\\[0.15cm]
        \STATE\quad Obtain $(\vx_{k}, \vy_k)$ to be $\eps_k$-saddle point of $f_k$ by applying DIPPA.
        \STATE\quad $\tilde\vx_{k} = \vx_{k} + \theta (\vx_{k} - \vx_{k-1})$.
        \STATE\textbf{end for} \\[0.1cm]
        \STATE \textbf{Output:} $\vx_K, \vy_K$.
    \end{algorithmic}
\end{algorithm}

We formally state the convergence rate in the following theorem.
\begin{thm}\label{thm:dippa}
    Assume that $g(\vx)$ is $L_x$-smooth and $\mu_x$-strongly convex, $h(\vy)$ is $L_y$-smooth and $\mu_y$-strongly convex and $L_x = L_y$. The total queries to Oracle (\ref{eq:oracle}) needed by Algorithm \ref{algo:catalyst-dippa} to produce $\eps$-saddle point of $f(\vx, \vy) = g(\vx) + \inner{\vx}{\mA \vy} - h(\vy)$ is at most
    \begin{align*}
        \tilde\gO\left( \left(\frac{\norm{\mA}}{\sqrt{\mu_x \mu_y}} + \sqrt[4]{\kappa_x \kappa_y (\kappa_x + \kappa_y)}\right) \log\left(\frac{1}{\eps}\right) \right),
    \end{align*}
    where $\kappa_x = L_x / \mu_x, \kappa_y = L_y / \mu_y$ and the notation $\tilde\gO$ have omitted some logarithmic factors depending on $\kappa_x, \kappa_y$ and $\frac{\norm{A}}{\sqrt{\mu_x \mu_y}}$.
\end{thm}

\section{Conclusion}
\label{sec:conclusion}
In this paper, we have proposed a novel algorithm DIPPA to solve bilinear saddle point problems.
Our method does not need any additional information about proximal operation of $g, h$ and achieves a tight dependency on the coupling condition number.
There is still a gap between the upper bounds and lower bounds of first-order algorithms for solving bilinear saddle point problems. We wish our technique can be used in a more general case other than the bilinear case.

\bibliographystyle{icml2021}
\bibliography{reference.bib}

\newpage
\appendix
\onecolumn
\section{Technique Lemmas}

We first present some equivalent statements of the definition of smoothness.
\begin{lemma}\label{lem:smooth:def}
    Let $\varphi$ be convex on $\sR^d$. then following conditions below, holding for all $\vx_1, \vx_2 \in \sR^d$, are equivalent:
    \begin{enumerate}[label=(\roman*)]
        \item $\norm{\nabla \varphi(\vx_1) - \nabla \varphi(\vx_2)} \le \ell \norm{\vx_1 - \vx_2}$, 
        \item $\varphi(\vx_2) - \varphi(\vx_1) - \inner{\nabla \varphi(\vx_1)}{\vx_2 - \vx_1} \le \frac{\ell}{2} \norm{\vx_1 - \vx_2}^2$, 
        \item $\inner{\nabla \varphi(\vx_1) - \nabla \varphi(\vx_2)}{\vx_1 - \vx_2} \ge \frac{1}{\ell} \norm{\nabla \varphi(\vx_1) - \nabla \varphi(\vx_2)}^2$.
    \end{enumerate}
    
\end{lemma}

\begin{proof}
    $ (i) \Rightarrow (ii)$: 
    Just note that
    \begin{align*}
        \varphi(\vx_2) - \varphi(\vx_1) - \inner{\nabla \varphi(\vx_1)}{\vx_2 - \vx_1} &= \int_0^1 \inner{\nabla \varphi(\vx_1 + t (\vx_2 - \vx_1)) - \nabla \varphi(\vx_1)}{\vx_2 - \vx_1} d t \\
        &\le \int_0^1 \ell t \norm{\vx_2 - \vx_1}^2 d t = \frac{\ell}{2} \norm{\vx_2 - \vx_1}^2,
    \end{align*}
    where the inequality follows from $(i)$ and Cauchy–Schwarz inequality. 
    
    $(ii) \Rightarrow (iii)$: 
    Consider the function $\psi(\vx) = \varphi(\vx) - \inner{\nabla \varphi(\vx_1)}{\vx}$ defined on $\sR^d$. \\
    It is easy to check that $\psi$ is convex and satisfies condition $(ii)$. Furthermore, the optimal point of $\psi$ is $\vx_1$, which implies
    \begin{equation}\label{eq:smooth:1}
    \begin{aligned}
        \psi(\vx_1) = \min_{\vx \in \sR^d} \psi(\vx) &\le \min_{\vx \in \sR^d} \left\{ \psi(\vx_2) + \inner{\nabla \psi(\vx_2)}{\vx - \vx_2} + \frac{\ell}{2} \norm{\vx - \vx_2}^2 \right\} \\
        &= \psi(\vx_2) - \frac{1}{2\ell} \norm{\nabla \psi(\vx_2)}^2,
    \end{aligned}
    \end{equation}
    where the optimal point of the second problem is $\vx_2 - \frac{1}{\ell} {\nabla \varphi(\vx_2)}$. 
    
    Following from the definition of $\psi$ and Equation (\ref{eq:smooth:1}), we have
    \begin{align}
        \notag \varphi(\vx_1) - &\inner{\nabla \varphi(\vx_1)}{\vx_1} \le \varphi(\vx_2) - \inner{\nabla \varphi(\vx_1)}{\vx_2} - \frac{1}{2 \ell} \norm{\nabla \varphi(\vx_2) - \nabla \varphi(\vx_1)}^2, ~~\text{i.e.,} \\
        \label{eq:smooth:2} &\varphi(\vx_2) - \varphi(\vx_1) - \inner{\nabla \varphi(\vx_1)}{\vx_2 - \vx_1} \ge \frac{1}{2 \ell} \norm{\nabla \varphi(\vx_2) - \nabla \varphi(\vx_1)}^2.
    \end{align}
    
    Similarly, there also holds
    \begin{align}\label{eq:smooth:3}
        \varphi(\vx_1) - \varphi(\vx_2) - \inner{\nabla \varphi(\vx_2)}{\vx_1 - \vx_2} \ge \frac{1}{2 \ell} \norm{\nabla \varphi(\vx_2) - \nabla \varphi(\vx_1)}^2.
    \end{align}
    
    Adding both sides of Equation (\ref{eq:smooth:2}) and (\ref{eq:smooth:3}) together, we know that $\varphi$ satisfies condition $(iii)$.
    
    $(iii) \Rightarrow (i)$:
    By Cauchy–Schwarz inequality and $(ii)$, we have
    \begin{align*}
        \norm{\nabla \varphi(\vx_1) - \nabla \varphi(\vx_2)}^2 &\le \ell \inner{\nabla \varphi(\vx_1) - \nabla \varphi(\vx_2)}{\vx_1 - \vx_2} \\
        &\le \ell \norm{\nabla \varphi(\vx_1) - \nabla \varphi(\vx_2)} \norm{\vx_1 - \vx_2},
    \end{align*}
    which is our desired result.
\end{proof}

Now, we are ready to prove Lemma \ref{lem:smooth}.
\begin{proof}[Proof of Lemma \ref{lem:smooth}]
    Consider function $\psi(\vx) = \varphi(\vx) - \frac{\mu}{2} \norm{\vx}^2$. By $\mu$-strongly convexity of $\varphi$, for any $\vx_1, \vx_2 \in \sR^d$, we know that 
    \begin{align*}
        &\quad \psi(\vx_2) - \psi(\vx_1) - \inner{\nabla \psi(\vx_1)}{\vx_2 - \vx_1} \\
        &= \varphi(\vx_2) - \varphi(\vx_1) - \inner{\nabla \varphi(\vx_1)}{\vx_2 - \vx_1} - \frac{\mu}{2} \left( \norm{\vx_2}^2 - \norm{\vx_1}^2 - 2\inner{\vx_1}{\vx_2 - \vx_1} \right) \\
        &= \varphi(\vx_2) - \varphi(\vx_1) - \inner{\nabla \varphi(\vx_1)}{\vx_2 - \vx_1} - \frac{\mu}{2} \norm{\vx_2 - \vx_1}^2 \ge 0,
    \end{align*}
    which implies that $\psi$ is convex. 
    
    On the other hand, by $\ell$-smoothness of $\varphi$ and Condition $(ii)$ in Lemma \ref{lem:smooth:def}, there holds
    \begin{align*}
        &\quad \psi(\vx_2) - \psi(\vx_1) - \inner{\nabla \psi(\vx_1)}{\vx_2 - \vx_1} \\
        &= \varphi(\vx_2) - \varphi(\vx_1) - \inner{\nabla \varphi(\vx_1)}{\vx_2 - \vx_1} - \frac{\mu}{2} \norm{\vx_2 - \vx_1}^2 \\
        &\le \frac{\ell - \mu}{2} \norm{\vx_2 - \vx_1}^2,
    \end{align*}
    which implies that $\psi$ is $(\ell - \mu)$-smooth.
    Consequently, following from Condition $(iii)$ in Lemma \ref{lem:smooth:def}, we have
    \begin{align*}
        \inner{\nabla \psi(\vx_1) - \nabla \psi(\vx_2)}{\vx_1 - \vx_2} \ge \frac{1}{\ell - \mu} \norm{\nabla \psi(\vx_1) - \nabla \psi(\vx_2)}^2, ~~~~\text{i.e.,} \\
        \inner{\nabla \varphi(\vx_1) - \nabla \varphi(\vx_2)}{\vx_1 - \vx_2} - \mu \norm{\vx_1 - \vx_2}^2 \ge \frac{1}{\ell - \mu} \norm{\nabla \varphi(\vx_1) - \nabla \varphi(\vx_2) - \mu(\vx_1 - \vx_2)}^2.
    \end{align*}
    By rearranging above inequality, we get that
    \begin{align*}
        \left(1 + \frac{2\mu}{\ell - \mu}\right)\inner{\nabla \varphi(\vx_1) - \nabla \varphi(\vx_2)}{\vx_1 - \vx_2} 
        \ge \left(\mu + \frac{\mu^2}{\ell - \mu}\right) \norm{\vx_1 - \vx_2}^2 + \frac{1}{\ell - \mu} \norm{\nabla \varphi(\vx_1) - \nabla \varphi(\vx_2)}^2,
    \end{align*}
    that is 
    \begin{align*}
        \frac{\ell + \mu}{\ell - \mu} \inner{\nabla \varphi(\vx_1) - \nabla \varphi(\vx_2)}{\vx_1 - \vx_2} 
        \ge \frac{\ell \mu}{\ell - \mu} \norm{\vx_1 - \vx_2}^2 + \frac{1}{\ell - \mu} \norm{\nabla \varphi(\vx_1) - \nabla \varphi(\vx_2)}^2.
    \end{align*}
\end{proof}

We then show Lipschitz continuity of the proximal operator with respect to strongly convex functions. 
\begin{lemma}\label{lem:prox:Lipschitz}
    Let $\varphi$ be convex on $\gX$. For all $\vx_1, \vx_2 \in \gX$, define 
    \begin{align*}
        \vu_i = \argmin_{\vu \in \gX} \varphi(\vu) + \frac{1}{2} \norm{\vu - \vx_i}, ~~ i = 1,2.
    \end{align*}
    Then there holds
    \begin{align*}
        \norm{\vu_1 - \vu_2} \le \norm{\vx_1 - \vx_2}.
    \end{align*}
\end{lemma}
\begin{proof}
    By strongly convexity of the functions $\Phi_i(\vu) \triangleq \varphi(\vu) + \frac{1}{2} \norm{\vu - \vx_i}^2$, we have
    \begin{align*}
        \varphi(\vu_2) + \frac{1}{2} \norm{\vu_2 - \vx_1}^2 &\ge \varphi(\vu_1) + \frac{1}{2} \norm{\vu_1 - \vx_1}^2 + \frac{1}{2} \norm{\vu_1 - \vu_2}^2, \\
        \varphi(\vu_1) + \frac{1}{2} \norm{\vu_1 - \vx_2}^2 &\ge \varphi(\vu_2) + \frac{1}{2} \norm{\vu_2 - \vx_2}^2 + \frac{1}{2} \norm{\vu_1 - \vu_2}^2.
    \end{align*}
    
    With adding both side of above two inequalities, we obtain that
    \begin{align*}
        \frac{1}{2} \norm{\vu_2 - \vx_1}^2 + \frac{1}{2} \norm{\vu_1 - \vx_2}^2 &\ge \frac{1}{2} \norm{\vu_1 - \vx_1}^2 + \frac{1}{2} \norm{\vu_2 - \vx_2}^2 + \norm{\vu_1 - \vu_2}^2, ~~\text{i.e.,} \\
        -\inner{\vu_2}{\vx_1} - \inner{\vu_1}{\vx_2} &\ge -\inner{\vu_1}{\vx_1} -\inner{\vu_2}{\vx_2} + \norm{\vu_1 - \vu_2}^2, ~~\text{i.e.,} \\
        \inner{\vu_1 - \vu_2}{\vx_1 - \vx_2} &\ge \norm{\vu_1 - \vu_2}^2.
    \end{align*}
    Then following from Cauchy–Schwarz inequality, there holds 
    \begin{align*}
        \norm{\vu_1 - \vu_2}^2 \le \norm{\vu_1 - \vu_2} \norm{\vx_1 - \vx_2},
    \end{align*}
    which implies that
    \begin{align*}
        \norm{\vu_1 - \vu_2} \le \norm{\vx_1 - \vx_2}.
    \end{align*}
\end{proof}

\section{Proof of Theorem \ref{thm:apfb}}
\label{app:apfb}
\begin{proof}
    Note that 
    \begin{align*}
        \vy_k = \argmin_{\vy} H_k(\vy) \triangleq h(\vy) + \frac{1}{2 \sigma} \norm{\vy - \vy_{k-1} - \sigma \mA^{\top} \tilde\vx_{k-1}}^2.
    \end{align*}
    By $(\mu_y + 1/\sigma)$-strongly convexity of $H_k$, we know that
    \begin{align*}
        &\quad h(\vy^*) + \frac{1}{2\sigma} \norm{\vy^* - \vy_{k-1} - \sigma \mA^{\top} \tilde\vx_{k-1}}^2 \\
        &\ge h(\vy_k) + \frac{1}{2\sigma} \norm{\vy_k - \vy_{k-1} - \sigma \mA^{\top} \tilde\vx_{k-1}}^2 + \left( \frac{\mu_y}{2} + \frac{1}{2\sigma}\right) \norm{\vy_k - \vy^*}^2,
    \end{align*}
    that is 
    \begin{align}
        \notag &\quad h(\vy^*) + \frac{1}{2\sigma} \norm{\vy_{k-1} - \vy^*}^2 + \inner{\vy_k - \vy^*}{\mA^{\top} \tilde\vx_{k-1}} \\
        \label{apfb:proof:1} &\ge h(\vy_k) + \frac{1}{2\sigma} \norm{\vy_k - \vy_{k-1}}^2 + \left( \frac{\mu_y}{2} + \frac{1}{2\sigma}\right) \norm{\vy_k - \vy^*}^2.
    \end{align}
    
    Similarly, by 
    \begin{align*}
        \vx_k = \argmin_{\vx} g(\vx) + \frac{1}{2 \gamma} \norm{\vx - \vx_{k-1} + \gamma \mA \vy_{k}}^2,
    \end{align*}
    we have
    \begin{align*}
        &\quad g(\vx^*) + \frac{1}{2\gamma} \norm{\vx^* - \vx_{k-1} + \gamma \mA \vy_k}^2 \\
        &\ge g(\vx_k) + \frac{1}{2\gamma} \norm{\vx_k - \vx_{k-1} + \gamma \mA \vy_k}^2 + \left( \frac{\mu_x}{2} + \frac{1}{2\gamma}\right) \norm{\vx_k - \vx^*}^2,
    \end{align*}
    which implies 
    \begin{align}
        \notag &\quad g(\vx^*) + \frac{1}{2\gamma} \norm{\vx_{k-1} - \vx^*}^2 - \inner{\vx_k - \vx^*}{\mA \vy_k} \\
        \label{apfb:proof:2} &\ge g(\vx_k) + \frac{1}{2\gamma} \norm{\vx_k - \vx_{k-1}}^2 + \left( \frac{\mu_x}{2} + \frac{1}{2\gamma}\right) \norm{\vx_k - \vx^*}^2.
    \end{align}
    
    Then we add both sides of the inequalities (\ref{apfb:proof:1}) and (\ref{apfb:proof:2}). Thus we have
    \begin{equation}\label{apfb:proof:3}
    \begin{aligned}
        &\quad \frac{1}{2\gamma} \norm{\vx_{k-1} - \vx^*}^2 + \frac{1}{2\sigma} \norm{\vy_{k-1} - \vy^*}^2 \\
        &\ge \left( \frac{\mu_x}{2} + \frac{1}{2\gamma}\right) \norm{\vx_k - \vx^*}^2 + \left( \frac{\mu_y}{2} + \frac{1}{2\sigma}\right) \norm{\vy_k - \vy^*}^2 \\
        &\quad + \frac{1}{2\gamma} \norm{\vx_k - \vx_{k-1}}^2 + \frac{1}{2\sigma} \norm{\vy_k - \vy_{k-1}}^2 \\
        &\quad + g(\vx_k) + h(\vy_k) - g(\vx^*) - h(\vy^*) + \inner{\vx_k - \vx^*}{\mA \vy_k} - \inner{\vy_k - \vy^*}{\mA^{\top} \tilde\vx_{k-1}}.
    \end{aligned}
    \end{equation}
    
    Observe that
    \begin{align}\label{apfb:proof:4}
        f(\vx_k, \vy^*) - f(\vx^*, \vy_k) 
        = g(\vx_k) + \inner{\vx_k}{\mA \vy^*} - h(\vy^*) - g(\vx^*) - \inner{\vx^*}{\mA \vy_k} + h(\vy_k).
    \end{align}
    Plugging Equality (\ref{apfb:proof:4}) into Inequality (\ref{apfb:proof:3}), we have
    \begin{equation}\label{apfb:proof:5}
    \begin{aligned}
        &\quad \frac{1}{2\gamma} \norm{\vx_{k-1} - \vx^*}^2 + \frac{1}{2\sigma} \norm{\vy_{k-1} - \vy^*}^2 \\
        &\ge \left( \frac{\mu_x}{2} + \frac{1}{2\gamma}\right) \norm{\vx_k - \vx^*}^2 + \left( \frac{\mu_y}{2} + \frac{1}{2\sigma}\right) \norm{\vy_k - \vy^*}^2 
        + \frac{1}{2\gamma} \norm{\vx_k - \vx_{k-1}}^2 + \frac{1}{2\sigma} \norm{\vy_k - \vy_{k-1}}^2 \\
        &\quad + f(\vx_k, \vy^*) - f(\vx^*, \vy_k) + \inner{\vx_k - \tilde\vx_{k-1}}{\mA (\vy_k - \vy^*)}.
    \end{aligned}
    \end{equation}
    
    With recalling the definition of $\tilde\vx_{k-1}$, the last term of Inequality (\ref{apfb:proof:5}) can be rewritten as
    \begin{equation}\label{apfb:proof:6}
    \begin{aligned}
        &\quad \inner{\vx_k - \tilde\vx_{k-1}}{\mA (\vy_k - \vy^*)} \\
        &= \inner{\vx_k - \vx_{k-1} - \theta (\vx_{k-1} - \vx_{k-2})}{\mA(\vy_k - \vy^*)} \\
        &= \inner{\vx_k - \vx_{k-1}}{\mA(\vy_k - \vy^*)} - \theta \inner{\vx_{k-1} - \vx_{k-2}}{\mA(\vy_{k-1} - \vy^*)} - \theta \inner{\vx_{k-1} - \vx_{k-2}}{\mA(\vy_k - \vy_{k-1})}.
    \end{aligned}
    \end{equation}
    
    Furthermore, we have
    \begin{equation}\label{apfb:proof:7}
    \begin{aligned}
        &\quad - \theta \inner{\vx_{k-1} - \vx_{k-2}}{\mA(\vy_k - \vy_{k-1})} \\
        &\ge -\theta \norm{\mA} \norm{\vx_{k-1} - \vx_{k-2}} \norm{\vy_k - \vy_{k-1}} \\
        &\ge -\frac{\theta}{2} \norm{\mA} \sqrt{\frac{\mu_x}{\mu_y}} \norm{\vx_{k-1} - \vx_{k-2}}^2 -\frac{\theta}{2} \norm{\mA} \sqrt{\frac{\mu_y}{\mu_x}} \norm{\vy_k - \vy_{k-1}}^2 \\
        &\ge - \frac{\theta}{2\gamma}\norm{\vx_{k-1} - \vx_{k-2}}^2 - \frac{1}{2\sigma} \norm{\vy_k - \vy_{k-1}}^2,
    \end{aligned}
    \end{equation}
    where we have recalled that $\gamma = \frac{1}{\norm{\mA}} \sqrt{\frac{\mu_y}{\mu_x}}, \sigma = \frac{1}{\norm{\mA}} \sqrt{\frac{\mu_x}{\mu_y}}$ and $\theta < 1$. 
    Similarly, there also holds
    \begin{align}\label{apfb:proof:8}
        \inner{\vx_k - \vx_{k-1}}{\mA(\vy_k - \vy^*)} \ge - \frac{1}{2\gamma} \norm{\vx_k - \vx_{k-1}}^2 - \frac{1}{2\sigma} \norm{\vy_k - \vy^*}^2.
    \end{align}
    
    Plugging Equation (\ref{apfb:proof:6}) and (\ref{apfb:proof:7}) into Inequality (\ref{apfb:proof:5}), we know that
    \begin{align*}
        &\quad \frac{1}{2\gamma} \norm{\vx_{k-1} - \vx^*}^2 + \frac{1}{2\sigma} \norm{\vy_{k-1} - \vy^*}^2 + \frac{\theta}{2\gamma}\norm{\vx_{k-1} - \vx_{k-2}}^2 + \theta \inner{\vx_{k-1} - \vx_{k-2}}{\mA(\vy_{k-1} - \vy^*)} \\
        &\ge \left( \frac{\mu_x}{2} + \frac{1}{2\gamma}\right) \norm{\vx_k - \vx^*}^2 + \left( \frac{\mu_y}{2} + \frac{1}{2\sigma}\right) \norm{\vy_k - \vy^*}^2 
        + \frac{1}{2\gamma} \norm{\vx_k - \vx_{k-1}}^2 \\
        &\quad + f(\vx_k, \vy^*) - f(\vx^*, \vy_k) + \inner{\vx_k - \vx_{k-1}}{\mA(\vy_k - \vy^*)}. 
    \end{align*}

    Therefore we have
    \begin{align*}
        &\quad \frac{\mu_x}{2} \norm{\vx_k - \vx^*}^2 + \frac{\mu_y}{2} \norm{\vy_k - \vy^*}^2 \\
        &\le \left( \frac{\mu_x}{2} + \frac{1}{2\gamma}\right) \norm{\vx_k - \vx^*}^2 + \left( \frac{\mu_y}{2} + \frac{1}{2\sigma}\right) \norm{\vy_k - \vy^*}^2 
        + \frac{1}{2\gamma} \norm{\vx_k - \vx_{k-1}}^2 + \inner{\vx_k - \vx_{k-1}}{\mA(\vy_k - \vy^*)} \\
        &\le \theta^{k} \left( \left( \frac{\mu_x}{2} + \frac{1}{2\gamma}\right) \norm{\vx_0 - \vx^*}^2 + \left( \frac{\mu_y}{2} + \frac{1}{2\sigma}\right) \norm{\vy_0 - \vy^*}^2 \right) \\
        &= \frac{\norm{\mA} \theta^{k-1} }{2}  \left( \sqrt{\frac{\mu_x}{\mu_y}}\norm{\vx_0 - \vx^*}^2 + \sqrt{\frac{\mu_y}{\mu_x}}\norm{\vy_0 - \vy^*}^2  \right).
    \end{align*}
\end{proof}
\section{Proof of Theorem \ref{thm:dppa}}
\begin{proof}
    Since $(\vx^*, \vy^*)$ is the saddle point of $f$, there holds $\nabla f(\vx, \vy) = 0$, that is 
    \begin{equation}\label{eq:dppa:saddle}
    \begin{aligned}
        \nabla g(\vx^*) + \mA \vy^* = 0, ~~~
        \nabla h(\vy^*) - \mA^{\top} \vx^* = 0.
    \end{aligned}
    \end{equation}
    Note that $(\vx_k, \vy_k) = \prox_{\alpha \mA} (2 \tilde\vx_k - \vz_k, 2\tilde\vy_k - \vw_k)$ which implies that
    \begin{align*}
        (\vx_k, \vy_k) = \arg\min_{\vx}\max_{\vy} \tilde{f}_k(\vx, \vy) \triangleq \frac{1}{2} \norm{\vx - 2 \tilde\vx_k + \vz_k}^2 - \frac{1}{2} \norm{\vy - 2\tilde\vy_k + \vw_k} + \inner{\vx}{\mA \vy}.
    \end{align*}
    Hence, we have $\nabla \tilde{f}_k (\vx_k, \vy_k) = 0$, that is 
    \begin{equation}\label{eq:dppa:prox2}
    \begin{aligned}
        \vx_k - (2 \tilde\vx_k - \vz_k) + \alpha \mA \vy_k = 0, \\
        \vy_k - (2\tilde\vy_k - \vw_k) - \alpha \mA^{\top} \vx_k = 0.
    \end{aligned}
    \end{equation}
    Similarly, according to $(\tilde\vx_k, \tilde\vy_k) = \prox_{\alpha (g - h)}(\vz_k, \vw_k)$, we have
    \begin{equation}\label{eq:dppa:prox1}
    \begin{aligned}
        \alpha \nabla g(\tilde\vx_k) + \tilde\vx_k - \vz_k = 0, \\
        \alpha \nabla h(\tilde\vy_k) + \tilde\vy_k - \vw_k = 0.
    \end{aligned}
    \end{equation}
    
    Therefore, we can conclude that 
    \begin{align*}
        &\quad \norm{\vz_{k+1} - \vz^*}^2 + \norm{\vw_{k+1} - \vw^*}^2 \\
        &= \norm{\vx_k - \vx^* - \alpha \mA (\vy_k - \vy^*)}^2 + \norm{\vy_k - \vy^* + \alpha \mA^{\top} (\vx_k - \vx^*)}^2 \\
        &= \norm{\vx_k - \vx^*}^2 + \norm{\vy_k - \vy^*}^2 + \alpha^2 \norm{\mA (\vy_k - \vy^*)}^2 + \alpha^2 \norm{\mA^{\top} (\vx_k - \vx^*)}^2 \\
        &= \norm{\vx_k - \vx^* + \alpha \mA (\vy_k - \vy^*)}^2 + \norm{\vy_k - \vy^* - \alpha \mA^{\top} (\vx_k - \vx^*)}^2 \\
        &= \norm{2 \tilde\vx_k - \vz_k -\vx^* - \alpha \mA \vy^* }^2 + \norm{2 \tilde\vy_k - \vw_k - \vy^* + \alpha \mA^{\top} \vx^*}^2 \\
        &= \norm{\tilde\vx_k - \alpha \nabla g(\tilde\vx_k) - \vx^* + \alpha \nabla g(\vx^*)}^2 + \norm{\tilde\vy_k - \alpha \nabla h(\tilde\vy_k) - \vy^* + \alpha \nabla h(\vy^*)}^2,
    \end{align*}
    where the forth equality is based on Equation (\ref{eq:dppa:prox2}) and the last equality follows from Equation (\ref{eq:dppa:saddle}) and (\ref{eq:dppa:prox1}). 
    
    On the other hand, by Equation (\ref{eq:dppa:saddle}) and (\ref{eq:dppa:prox1}), we also have
    \begin{align*}
        \norm{\vz_{k} - \vz^*}^2 + \norm{\vw_{k} - \vw^*}^2 
        = \norm{\tilde\vx_k + \alpha \nabla g(\tilde\vx_k) - \vx^* - \alpha \nabla g(\vx^*)}^2 + \norm{\tilde\vy_k + \alpha \nabla h(\tilde\vy_k) - \vy^* - \alpha \nabla h(\vy^*)}^2
    \end{align*}
    
    Now, We only need to prove that 
    \begin{align}\label{dppa:proof}
        \norm{\tilde\vx_k - \alpha \nabla g(\tilde\vx_k) - \vx^* + \alpha \nabla g(\vx^*)}^2 \le \eta \norm{\tilde\vx_k + \alpha \nabla g(\tilde\vx_k) - \vx^* - \alpha \nabla g(\vx^*)}^2.
    \end{align}
    In fact, above inequality is equivalent to 
    \begin{align}\label{dppa:proof:2}
        (1 - \eta) \left(\norm{\tilde\vx_k - \vx^*}^2 + \alpha^2 \norm{\nabla g(\tilde\vx_k) - \nabla g(\vx^*)}^2 \right) 
        \le 2(1 + \eta)\alpha \inner{\tilde\vx_k - \vx^*}{\nabla g(\tilde\vx_k) - \nabla g(\vx^*)}.
    \end{align}
    Recalling the definition of $\eta$, we have $\frac{1 - \eta}{2(1 + \eta)} = \frac{\sqrt{L \mu}}{L + \mu}$. 
    Then together with $\alpha = 1/\sqrt{L \mu}$, inequality (\ref{dppa:proof:2}) is just
    \begin{align*}
        \frac{L \mu}{L + \mu} \norm{\tilde\vx_k - \vx^*}^2 + \frac{1}{L + \mu} \norm{\nabla g(\tilde\vx_k) - \nabla g(\vx^*)}^2 \le \inner{\tilde\vx_k - \vx^*}{\nabla g(\tilde\vx_k) - \nabla g(\vx^*)},
    \end{align*}
    which holds according to Lemma \ref{lem:smooth}.
    \end{proof}
    
\begin{remark}
    By the proof of inequality (\ref{dppa:proof}), there also holds
    \begin{align}\label{eq:dppa:3}
        \norm{2 \tilde\vx_k - \vz_k - 2 \vx^* + \vz^* }^2 \le \eta \norm{\vz_{k} - \vz^*}^2.
    \end{align}
\end{remark}
\section{Proof of Theorem \ref{thm:dippa-out}}
\begin{proof}
    Denote $\tilde\vx_{k}^* = \argmin_{\vx} G_k(\vx)$. 
    Note that 
    \begin{align*}
        &\quad \norm{2\tilde\vx_{k} - \vz_k - 2\vx^* + \vz^*}^2 \\
        &\le 4 (1 + \beta) \norm{\tilde\vx_k - \tilde\vx_k^*}^2 + (1 + 1/\beta) \norm{2 \tilde\vx_k^* - \vz_k - 2 \vx^* + \vz^*}^2 \\
        &\le 4 (1 + \beta) \norm{\tilde\vx_k - \tilde\vx_k^*}^2 + (1 + 1/\beta) \left(\frac{\sqrt{\kappa} - 1}{\sqrt{\kappa} + 1}\right)^2 \norm{\vz_k -\vz^*}^2 \\
        &\le 2 (\sqrt{\kappa} + 1) \norm{\tilde\vx_k - \tilde\vx_k^*}^2 + \frac{\sqrt{\kappa} - 1}{\sqrt{\kappa} + 1} \norm{\vz_k -\vz^*}^2,
    \end{align*}
    where $\beta = \frac{\sqrt{\kappa} - 1}{2}$ and the second inequality is according to Equation (\ref{eq:dppa:3}).
    
    Observe that $G_k$ is $(\mu + 1/\alpha)$-strongly convex, hence we have
    \begin{small}
    \begin{align*}
        \left( \frac{\mu + \sqrt{L\mu}}{2} \right) \norm{\tilde\vx_{k} - \tilde\vx_{k}^*}^2 \le G_k(\tilde\vx_{k}) - G_k(\tilde\vx_{k}^*) \le \eps_k.
    \end{align*}
    \end{small}
    Therefore, there holds
    \begin{equation}\label{dippa:proof:out}
    \begin{aligned}
        \norm{2\tilde\vx_{k} - \vz_k - 2\vx^* + \vz^*}^2 
        \le \frac{4}{\mu} \eps_k + \frac{\sqrt{\kappa} - 1}{\sqrt{\kappa} + 1} \norm{\vz_k -\vz^*}^2. 
    \end{aligned}
    \end{equation}

    On the other hand, let $(\vx_{k}^*, \vy_{k}^*)$ be the saddle point of $\tilde{f}_k$, which satisfies
    \begin{align}\label{eq:dippa:prox2}
        \begin{cases}
            \vx_{k}^* + \alpha \mA \vy_{k}^* = 2\tilde\vx_{k} - \vz_k, \\
            \vy_{k}^* - \alpha \mA^{\top} \vx_{k}^* = 2\tilde\vy_{k} - \vw_k.
        \end{cases}
    \end{align}
    Then we have 
    \begin{align*}
        &\quad \norm{\vx_{k} - \vx^* + \alpha \mA (\vy_{k} - \vy^*)}^2 \\
        &\le (1 + \sqrt{\kappa}) \norm{\vx_{k} - \vx_{k}^* + \alpha \mA (\vy_{k} - \vy_{k}^*)}^2 
        + (1 + 1/\sqrt{\kappa}) \norm{\vx^*_{k} - \vx^* + \alpha \mA (\vy^*_{k} - \vy^*)}^2 \\
        &\le (1 + \sqrt{\kappa}) \left( 2 \norm{\vx_{k} - \vx_{k}^*}^2 + \frac{2 \norm{\mA}^2}{L \mu} \norm{\vy_{k} - \vy_{k}^*}^2 \right) 
        + \frac{8}{\mu} \eps_k + \frac{\sqrt{\kappa} - 1}{\sqrt{\kappa}} \norm{\vz_k - \vz^*}^2,
    \end{align*}
    where the second inequality is according to Equation (\ref{dippa:proof:out}) and (\ref{eq:dippa:prox2}).

    Similarly, we also have
    \begin{align*}
        &\quad \norm{\vy_{k} - \vy^* - \alpha \mA^{\top} (\vx_{k} - \vx^*)}^2 \\
        &\le (1 + \sqrt{\kappa}) \left( 2 \norm{\vy_{k} - \vy_{k}^*}^2 + \frac{2 \norm{\mA}^2}{L \mu} \norm{\vx_{k} - \vx_{k}^*}^2 \right) 
        + \frac{8}{\mu} \eps_k + \frac{\sqrt{\kappa} - 1}{\sqrt{\kappa}} \norm{\vw_k - \vw^*}^2
    \end{align*}

    Therefore, we can conclude that 
    \begin{align*}
        &\quad \norm{\vz_{k+1} - \vz^*}^2 + \norm{\vw_{k+1} - \vw^*}^2 \\
        &= \norm{\vx_k - \vx^* - \alpha \mA (\vy_k - \vy^*)}^2 + \norm{\vy_k - \vy^* + \alpha \mA^{\top} (\vx_k - \vx^*)}^2 \\
        &= \norm{\vx_k - \vx^* + \alpha \mA (\vy_k - \vy^*)}^2 + \norm{\vy_k - \vy^* - \alpha \mA^{\top} (\vx_k - \vx^*)}^2 \\
        &\le 2 (1 + \sqrt{\kappa}) \left(1 + \frac{\norm{\mA}^2}{L \mu}\right) \delta_k + \frac{16}{\mu} \eps_k + \frac{\sqrt{\kappa} - 1}{\sqrt{\kappa}} \left(\norm{\vz_k - \vz^*}^2 + \norm{\vw_k - \vw^*}^2\right) \\
        &\le 2 C_0 (1 - \rho)^{k+1} + (1 - 2\rho) \left(\norm{\vz_{k} - \vz^*}^2 + \norm{\vw_{k} - \vw^*}^2\right),
    \end{align*}
    where we have recalled the definition of $\eps_k$ and $\delta_k$.
    
    Let $a_k = \norm{\vz_{k} - \vz^*}^2 + \norm{\vw_{k} - \vw^*}^2$. Then we have
    \begin{align*}
        \frac{a_{k+1}}{(1 - 2 \rho)^{k+1}} - \frac{a_{k}}{(1 - 2 \rho)^{k}} &\le 2 C_0 \left( \frac{1 - \rho}{1 - 2 \rho} \right)^{k+1}, ~~i.e., \\
        \frac{a_{k+1}}{(1 - 2 \rho)^{k+1}} - \frac{a_1}{1 - 2\rho} &\le 2 C_0 \sum_{i = 2}^{k+1} \left( \frac{1 - \rho}{1 - 2 \rho} \right)^{i} = 2 C_0 \left( \frac{1 - \rho}{1 - 2 \rho} \right)^2 \frac{\left( \frac{1 - \rho}{1 - 2 \rho} \right)^{k} - 1}{ \left( \frac{1 - \rho}{1 - 2 \rho} \right) - 1} \\
        & \le \frac{2 C_0 (1 - \rho)^2}{\rho(1 - 2 \rho)} \left( \frac{1 - \rho}{1 - 2 \rho} \right)^{k}.
    \end{align*}
    
    Consequently, we have
    \begin{align*}
        &\quad \norm{\vx_{K} - \vx^*}^2 + \norm{\vy_{K} - \vy^*}^2 
        \le \norm{\vz_{K+1} - \vz^*}^2 + \norm{\vw_{K+1} - \vw^*}^2 \\
        &\le \frac{2 C_0}{\rho} (1 - \rho)^{K + 2} + a_1 (1 - 2 \rho)^{K} 
        \le (4 C_0 \sqrt{\kappa} + a_1) (1 - \rho)^{K+1},
    \end{align*}
    where the last inequality is according to $(1 - 2 \rho)^{k} \le (1 - \rho)^{k+1}$ for $k \ge 1$ and $\rho = \frac{1}{2 \sqrt{\kappa}}$
    Then, together with 
    \begin{align*}
        a_1 \le \frac{\norm{\mA}^2 + L \mu}{L \mu} \left( \norm{\vx_{0} - \vx^*}^2 + \norm{\vy_{0} - \vy^*}^2 \right) 
        = \frac{C_0 (\norm{\mA}^2 + L \mu) }{L \mu}
    \end{align*}
    we obtain the desired result.
\end{proof}

\section{Proof of Lemma \ref{lem:dippa:inner}}
\begin{proof}
    Denote $\tilde\vx_{k}^* = \argmin_{\vx} G_k(\vx) = \argmin_{\vx} g(\vx) + \frac{1}{2\alpha} \norm{\vx - \vz_k}^2$. \\
    Observe that $\vx^* = \argmin_{\vx} g(\vx) + \frac{1}{2\alpha} \norm{\vx - \vz^*}^2$. 
    Then by Lemma \ref{lem:prox:Lipschitz}, we have
    \begin{align*}
        \norm{\tilde\vx_{k}^* - \vx^*}^2 \le \norm{\vz_k - \vz^*}^2.
    \end{align*}
    Hence, we have
    \begin{align*}
        \norm{\vx_{k-1} - \tilde\vx_{k}^*}^2 
        \le 2 \norm{\vx_{k-1} - \vx^*}^2 + 2 \norm{\tilde\vx_{k}^* - \vx^*}^2 
        \le 4 \left(\norm{\vz_{k} - \vz^*}^2 + \norm{\vw_k - \vw^*}^2\right).
    \end{align*}
    Note that the condition number of function $G_k$ is 
    \begin{align*}
        \frac{L + 1/\alpha}{\mu + 1/\alpha} = \sqrt{\kappa}.
    \end{align*}
    
    Suppose the sequence $\{\tilde\vx_{k, t}\}_{t = 0}^{K_1}$ is obtained by AGD for optimizing $G_k$ where $\tilde\vx_{k, 0} = \vx_{k -1}, \tilde\vx_{k, K_1} = \tilde\vx_k$. Then following from Theorem \ref{thm:agd}, there holds
    \begin{align*}
        G_k(\tilde\vx_k) - G_k(\tilde\vx^*) &\le \frac{L + \mu + 2/\alpha}{2} \norm{\vx_{k-1} - \tilde\vx_k^*}^2 \exp\left(- \frac{K_1}{\sqrt[4]{\kappa}}\right) \\
        &\le 2 C_0 (L + \mu + 2\sqrt{L \mu}) C  (1 - \rho)^k \frac{\mu(1 - \rho)}{32 C (\sqrt{L} + \sqrt{\mu})^2 } \\
        &\le \frac{C_0 \mu}{16} (1 - \rho)^{k+1} = \eps_k.
    \end{align*}
    
    Similarly, we also need to run AGD $K_1$ steps for optimizing $H_k$ with initial point $\vy_{k-1}$.
    
    Now, we turn to consider $\tilde{f}_k$. Let $(\vx_{k}^*, \vy_{k}^*)$ be the saddle point of $\tilde{f}_k$. Then we have
    \begin{align*}
        \norm{\vx^*_{k} - \vx^* + \alpha \mA (\vy^*_{k} - \vy^*)}^2 &\le \frac{8}{\mu} \eps_k + \frac{\sqrt{\kappa} - 1}{\sqrt{\kappa} + 1} \norm{\vz_k - \vz^*}^2, \\
        \norm{\vy^*_{k} - \vy^* - \alpha \mA^{\top} (\vx^*_{k} - \vx^*)}^2 &\le \frac{8}{\mu} \eps_k + \frac{\sqrt{\kappa} - 1}{\sqrt{\kappa} + 1} \norm{\vw_k - \vw^*}^2.
    \end{align*}
    
    Therefore, we have
    \begin{align*}
        &\quad \norm{\vx_{k-1} - \vx_{k}^*}^2 + \norm{\vy_{k-1} - \vy_k^*} \\
        &\le 2 \left(\norm{\vx_{k-1} - \vx^*}^2 + \norm{\vy_{k-1} - \vy^*}^2\right) + 2 \left(\norm{\vx_{k}^* - \vx^*}^2 + \norm{\vy_k^* - \vy^*}^2\right) \\
        &\le 4 \left(\norm{\vz_k - \vz^*}^2 + \norm{\vw_k - \vw^*}^2\right) + \frac{16}{\mu} \eps_k \\
        &\le C_0 (4 C (1 - \rho)^k + (1 - \rho)^{k+1}).
    \end{align*}
    
    Suppose the sequence $\{(\vx_{k, t}, \vy_{k, t})\}_{t = 0}^{K_2}$ is obtained by APFB for solving the subproblem (\ref{prob:sub2}) where $(\vx_{k, 0}, \vy_{k, 0}) = (\vx_{k -1}, \vy_{k-1}),  (\vx_{k, K_2}, \vy_{k, K_2}) = (\vx_k, \vy_k)$.
    
    Then by Theorem \ref{thm:apfb}, we have
    \begin{align*}
        \norm{\vx_k - \vx_k^*}^2 + \norm{\vy_k - \vy_k^*}^2 &\le \left( \frac{\norm{\mA}}{\norm{\mA} + \sqrt{L \mu}} \right)^{K_2 - 1} \left(\norm{\vx_{k-1} - \vx_k^*}^2 + \norm{\vy_{k-1} - \vy_k^*}^2\right) \\
        &\le C_0 \left(4 C (1 - \rho)^k + (1 - \rho)^{k+1}\right) \exp\left( -\frac{K_2 - 1}{\frac{\norm{\mA}}{\sqrt{L \mu}} + 1 }  \right) \\
        &\le 5 C_0 C (1 - \rho)^k \frac{L\mu (1 - \rho)}{20 C (1 + \sqrt{\kappa}) (L \mu + \norm{\mA}^2) } \le \delta_k.
    \end{align*}
\end{proof}

\section{Accelerated Inexact Proximal Forward Backward Algorithm}
\label{sec:aipfb}
In this section, we provide an ineaxct version of APFB, called Accelerated Inexact Proximal Forward Backward, in Algorithm \ref{algo:aipfb} for completeness. Similar to DIPPA, we employ AGD to solve subproblems. And a theoretical guarantee is given in following theorem.

\begin{thm}\label{thm:aipfb}
    Assume that $g(\vx)$ is $L_x$-smooth and $\mu_x$-strongly convex and $h(\vy)$ is $L_y$-smooth and $\mu_y$-strongly convex.
    The total queries to Oracle (\ref{eq:oracle}) needed by Algorithm \ref{algo:aipfb} to produce $\eps$-saddle point of $f(\vx, \vy) = g(\vx) + \inner{\vx}{\mA \vy} - h(\vy)$ is at most
    \begin{align*}
        \tilde\gO\left(\left( \frac{\norm{\mA}}{\sqrt{\mu_x \mu_y}} + \frac{\sqrt{\norm{\mA} (\kappa_x + \kappa_y)}}{\sqrt[4]{\mu_x \mu_y}} + \sqrt{\kappa_x + \kappa_y}\right)\log\left( \frac{\mu_x \norm{\vx_{0} - \vx^*}^2 + \mu_y \norm{\vy_{0} - \vy^*}^2}{\eps} \right)\right),
    \end{align*}
    where $\kappa_x = L_x / \mu_x$, $\kappa_y = L_y / \mu_y$ and the notation $\tilde\gO$ have omitted some logarithmic factors depending on $L_x$, $L_y$, $\norm{\mA}$, $\mu_x$ and $\mu_y$. 
\end{thm}

\begin{algorithm}[t]
    \caption{AIPFB}\label{algo:aipfb}
    \begin{algorithmic}[1]
        \STATE \textbf{Input:} function $g, h$, coupling matrix $\mA$, initial point $\vx_0, \vy_0$, strongly convex module $\mu_x, \mu_y$, run-time $T$, tolerance sequence $\{ \eps_k \}_{k \ge 1}$. \\[0.15cm]
        \STATE \textbf{Initialize:} $\tilde\vx_0 = \vx_0$, $\gamma = \frac{1}{\norm{\mA}}\sqrt{\frac{\mu_y}{\mu_x}}$, $\sigma = \frac{1}{\norm{\mA}}\sqrt{\frac{\mu_x}{\mu_y}}$ and $\theta = \frac{\norm{\mA}}{\sqrt{\mu_x \mu_y} + \norm{\mA}}$. \\[0.15cm]
        \STATE \textbf{for} $k = 1, \cdots, T$ \textbf{do}\\[0.15cm]
        \STATE\quad Let $h_k(\vy) = h(\vy) + \frac{1}{2 \sigma} \norm{\vy - \vy_{k-1} - \sigma \mA^{\top} \tilde\vx_{k-1}}^2 $. \\[0.15cm]
        \STATE\quad Find $\vy_k$ such that $h_k(\vy_{k}) - \min_{\vy} h_k(\vy) \le \eps_k.$ \\[0.15cm]
        \STATE\quad Let $g_k(\vx) = g(\vx) + \frac{1}{2\gamma} \norm{\vx - \vx_{k-1} + \gamma \mA \vy_k}^2 $. \\[0.15cm]
        \STATE\quad Find $\vx_k$ such that $g_k(\vx_{k}) - \min_{\vx} g_k(\vx) \le \eps_k.$ \\[0.15cm]
        \STATE\quad $\tilde\vx_k = \vx_k + \theta (\vx_k - \vx_{k-1})$. \\[0.15cm]
        \STATE\textbf{end for} \\[0.1cm]
        \STATE \textbf{Output:} $\vx_T, \vy_T$.
    \end{algorithmic}
\end{algorithm}

We first present the convergence rate of the outer loop of Algorithm \ref{algo:aipfb}.

\begin{lemma}\label{thm:aipfb-out}
    Assume that $g(\vx)$ is $L_x$-smooth and $\mu_x$-strongly convex and $h(\vy)$ is $L_y$-smooth and $\mu_y$-strongly convex.
    Set
    \begin{align*}
        \theta = \frac{\norm{\mA} }{ \sqrt{\mu_x \mu_y} + \norm{\mA} }, ~~~
        & \rho = \frac{ \sqrt{\mu_x \mu_y} }{ 2 \sqrt{\mu_x \mu_y} + 4 \norm{\mA} }, \\
        C_0 = \mu_x \norm{\vx_{0} - \vx^*}^2 + \mu_y \norm{\vy_{0} - \vy^*}^2, ~~~
        & \eps_k = \frac{C_0 \rho (1 - \theta)}{16}  (1 - \rho)^{k-1}.
    \end{align*}
    For $T \ge 2$, the output of Algorithm \ref{algo:aipfb} satisfies
    \begin{align*}
        \mu_x \norm{\vx_T - \vx^*}^2 + \mu_y \norm{\vy_T - \vy^*} 
        \le C (1 - \rho)^{T} \left( \mu_x \norm{\vx_{0} - \vx^*}^2 + \mu_y \norm{\vy_{0} - \vy^*}^2 \right) ,
    \end{align*}
    where
    \begin{align*}
    C = \frac{\norm{\mA}}{\sqrt{\mu_x \mu_y}}  + 1,
    \end{align*}
    and $(\vx^*, \vy^*)$ is the saddle point of the function $f(\vx, \vy) = g(\vx) + \vx^{\top} \mA \vy - h(\vy)$.
\end{lemma}

\begin{proof}
    Denote $\vy_k^* = \argmin_{\vy} h_k (\vy)$,
    $\vx_k^* = \argmin_{\vx} g_k (\vx)$, 
    $\vw_k = \vy_{k-1} + \sigma \mA^{\top} \tilde{\vx}_{k-1}$ and
    $\vz_k = \vx_{k-1} - \gamma \mA \vy_k$.
    Since $h_k(\vy)$ is $(\mu_y + 1 / \sigma)$-strongly convex and $h_k(\vy_k) - h_k(\vy_k^*) \le \eps_k$,
    we know that
    \begin{align*}
        h_k(\vy^*) 
        \ge h_k(\vy_k^*) + \left( \frac{\mu_y}{2} + \frac{1}{2 \sigma} \right) \norm{\vy^* - \vy_k^*}^2
        \ge h_k(\vy_k) + \left( \frac{\mu_y}{2} + \frac{1}{2 \sigma} \right) \norm{\vy^* - \vy_k^*}^2 - \eps_k.
    \end{align*}
    Equivalently, we have
    \begin{align*}
        h(\vy^*) 
        & \ge h(\vy_k) + \frac{1}{2 \sigma} \norm{\vy_k - \vw_k}^2 - \frac{1}{2 \sigma} \norm{\vy^* - \vw_k} +  \left( \frac{\mu_y}{2} + \frac{1}{2 \sigma} \right) \norm{\vy^* - \vy_k^*}^2 - \eps_k \\
        & = h(\vy_k) + \frac{1}{\sigma} \inner{\vy_k - \vw_k}{\vy_k - \vy^*} - \frac{1}{2 \sigma} \norm{\vy_k - \vy^*}^2 + \left( \frac{\mu_y}{2} + \frac{1}{2 \sigma} \right) \norm{\vy^* - \vy_k^*}^2 - \eps_k.
    \end{align*}
    On the other hand, note that
    \begin{align*}
        & \quad \left( \frac{\mu_y}{2} + \frac{1}{2 \sigma} \right) \norm{\vy^* - \vy_k^*}^2 - \frac{1}{2 \sigma} \norm{\vy_k - \vy^*}^2 \\
        & = \frac{\mu_y}{2} \norm{\vy_k - \vy^*}^2 - \left(\mu_y + \frac{1}{\sigma} \right) \inner{\vy_k - \vy^*}{\vy_k - \vy_k^*} + \left( \frac{\mu_y}{2} + \frac{1}{2 \sigma} \right) \norm{\vy_k - \vy_k^*}^2.
    \end{align*}
    Using Young's inequality yields
    \begin{align*}
        \left(\mu_y + \frac{1}{\sigma} \right) \inner{\vy_k - \vy^*}{\vy_k - \vy_k^*}
        \le \frac{\mu_y}{4} \norm{\vy_k - \vy^*}^2 + \left(\mu_y + \frac{1}{\sigma} \right) \left( 1 + \frac{1}{\mu_y \sigma} \right) \norm{\vy_k - \vy_k^*}^2.
    \end{align*}
    It follows that
    \begin{align*}
        h(\vy^*) 
        \ge h(\vy_k) + \frac{1}{\sigma} \inner{\vy_k - \vw_k}{\vy_k - \vy^*}  + \frac{\mu_y}{4} \norm{\vy_k - \vy^*}^2  - \left(\mu_y + \frac{1}{\sigma} \right) \left( \frac{1}{2} + \frac{1}{\mu_y \sigma} \right) \norm{\vy_k - \vy_k^*}^2 - \eps_k.
    \end{align*}
    By $(\mu_y + 1 / \sigma)$-strongly convexity of $h_k$, we have 
    \begin{align*}
        \left( \frac{\mu_y}{2} + \frac{1}{2 \sigma} \right) \norm{\vy_k - \vy_k^*}^2
        \le h_k(\vy_k) - h_k (\vy_k^*) \le \eps_k.
    \end{align*}
    Putting these pieces together yields
    \begin{align}
    \label{aipfb:proof:1}
        h(\vy^*) 
        \ge h(\vy_k) + \frac{1}{\sigma} \inner{\vy_k - \vw_k}{\vy_k - \vy^*}  + \frac{\mu_y}{4} \norm{\vy_k - \vy^*}^2 - \left( 2 + \frac{2}{\mu_y  \sigma}\right) \eps_k.
    \end{align}
    Plugging $\vw_k = \vy_{k-1} + \sigma \mA^{\top} \tilde{\vx}_{k-1}$ and 
    $ 
        2 \inner{\vy_k - \vy_{k-1}}{\vy_k - \vy^*} = \norm{\vy_k - \vy_{k-1}}^2 + \norm{\vy_k - \vy^*}^2 - \norm{\vy_{k-1} - \vy^*}^2
    $ 
    into Inequality (\ref{aipfb:proof:1}),
    we have
    \begin{equation}
    \label{aipfb:proof:2}
    \begin{aligned}
        & \quad h(\vy^*) + \frac{1}{2 \sigma} \norm{\vy_{k-1} - \vy^*}^2 + \inner{\vy_k - \vy^*}{\mA^{\top} \tilde{\vx}_{k-1}} \\
        & \ge h(\vy_k) + \frac{1}{2 \sigma} \norm{\vy_k - \vy_{k-1}}^2 + \left( \frac{\mu_y}{4} + \frac{1}{2 \sigma} \right) \norm{\vy_k - \vy^*}^2 - \left( 2 + \frac{2}{\mu_y  \sigma}\right) \eps_k.
    \end{aligned}
    \end{equation}
    Similarly, we can obtain
    \begin{equation}
    \label{aipfb:proof:3}
    \begin{aligned}
        & \quad g(\vx^*) + \frac{1}{2 \gamma} \norm{\vx_{k-1} - \vx^*}^2 - \inner{\vx_k - \vx^*}{\mA \vy_{k}} \\
        & \ge g(\vx_k) + \frac{1}{2 \gamma} \norm{\vx_k - \vx_{k-1}}^2 + \left( \frac{\mu_x}{4} + \frac{1}{2 \gamma} \right) \norm{\vx_k - \vx^*}^2 - \left( 2 + \frac{2}{\mu_x  \gamma}\right) \eps_k.
    \end{aligned}
    \end{equation}
    Observe that $\mu_x \gamma = \mu_y \sigma = \frac{\sqrt{\mu_x \mu_y} }{ \norm{\mA} } = 1 / \theta - 1$.
    Adding both sides of Inequalities (\ref{aipfb:proof:2}) and (\ref{aipfb:proof:3}) yields
    \begin{equation}\label{aipfb:proof:4}
    \begin{aligned}
        &\quad \frac{1}{2\gamma} \norm{\vx_{k-1} - \vx^*}^2 + \frac{1}{2\sigma} \norm{\vy_{k-1} - \vy^*}^2 \\
        &\ge \left( \frac{\mu_x}{4} + \frac{1}{2\gamma}\right) \norm{\vx_k - \vx^*}^2 + \left( \frac{\mu_y}{4} + \frac{1}{2\sigma}\right) \norm{\vy_k - \vy^*}^2 
        + \frac{1}{2\gamma} \norm{\vx_k - \vx_{k-1}}^2 + \frac{1}{2\sigma} \norm{\vy_k - \vy_{k-1}}^2 \\
        & \quad + g(\vx_k) + h(\vy_k) - g(\vx^*) - h(\vy^*) + \inner{\vx_k - \vx^*}{\mA \vy_k} - \inner{\vy_k - \vy^*}{\mA^{\top} \tilde\vx_{k-1}} - \frac{4}{1 - \theta}\, \eps_k.
    \end{aligned}
    \end{equation}
    Plugging Equations (\ref{apfb:proof:4}), (\ref{apfb:proof:6}) and (\ref{apfb:proof:7}) into Inequality (\ref{aipfb:proof:4}), we have
        \begin{align*}
        &\quad \frac{1}{2\gamma} \norm{\vx_{k-1} - \vx^*}^2 + \frac{1}{2\sigma} \norm{\vy_{k-1} - \vy^*}^2 + \frac{\theta}{2\gamma}\norm{\vx_{k-1} - \vx_{k-2}}^2 + \theta \inner{\vx_{k-1} - \vx_{k-2}}{\mA(\vy_{k-1} - \vy^*)} \\
        &\ge \left( \frac{\mu_x}{4} + \frac{1}{2\gamma}\right) \norm{\vx_k - \vx^*}^2 + \left( \frac{\mu_y}{4} + \frac{1}{2\sigma}\right) \norm{\vy_k - \vy^*}^2 
        + \frac{1}{2\gamma} \norm{\vx_k - \vx_{k-1}}^2 \\
        &\quad + f(\vx_k, \vy^*) - f(\vx^*, \vy_k) + \inner{\vx_k - \vx_{k-1}}{\mA(\vy_k - \vy^*)}  - \frac{4}{1 - \theta}\, \eps_k.
    \end{align*}
    By Definition \ref{def:saddle}, we have $ f(\vx_k, \vy^*) - f(\vx^*, \vy_k) \ge 0$. Recall that $\eps_k = \frac{C_0 \rho (1 - \theta)}{16}  (1 - \rho)^{k-1}$ where $ \rho = \frac{ \sqrt{\mu_x \mu_y} }{ 2 \sqrt{\mu_x \mu_y} + 4 \norm{\mA} } $.
    Denoting
    \[
    a_k = \left( \frac{\mu_x}{4} + \frac{1}{2\gamma}\right) \norm{\vx_k - \vx^*}^2 + \left( \frac{\mu_y}{4} + \frac{1}{2\sigma}\right) \norm{\vy_k - \vy^*}^2 + \frac{1}{2\gamma} \norm{\vx_k - \vx_{k-1}}^2 + \inner{\vx_k - \vx_{k-1}}{\mA(\vy_k - \vy^*)},
    \]
    we have
    \begin{align*}
        \frac{a_k}{ (1 - 2 \rho)^k } & \le \frac{ a_{k-1}}{ (1 - 2 \rho)^{k-1} } + \frac{C_0 \rho}{4 (1 - \rho) } \left( \frac{1 - \rho}{1 - 2 \rho} \right)^k, ~~i.e., \\
        \frac{a_k}{ (1 - 2 \rho)^k } & \le a_0 + \frac{C_0 \rho}{4 (1 - \rho) } \sum_{i=1}^k \left( \frac{1 - \rho}{1 - 2 \rho} \right)^i = a_0 + \frac{C_0 \rho}{4 (1 - \rho) }  \frac{1 - \rho}{1 - 2 \rho}  \frac{\left( \frac{1 - \rho}{1 - 2 \rho} \right)^{k} - 1}{ \left( \frac{1 - \rho}{1 - 2 \rho} \right) - 1} \\
        & \le a_0 + \frac{ C_0 }{ 4 } \left( \frac{1 - \rho}{1 - 2 \rho} \right)^{k}. 
    \end{align*}
    Moreover, Inequality (\ref{apfb:proof:8}) implies $a_T \ge \frac{\mu_x}{4} \norm{\vx_T - \vx^*}^2 + \frac{\mu_y}{4} \norm{\vy_T - \vy^*}^2 $. 
    Consequently, for $T \ge 2$ we have
    \begin{align*}
        & \quad \mu_x \norm{\vx_T - \vx^*}^2 + \mu_y \norm{\vy_T - \vy^*} \\
        & \le 4 a_0 (1 - 2 \rho)^T + C_0 (1 - \rho)^{T} \\
        & = (1 - 2 \rho)^T \left[ \left( {\mu_x} + \frac{2}{\gamma}\right) \norm{\vx_0 - \vx^*}^2 + \left( {\mu_y} + \frac{2}{\sigma}\right) \norm{\vy_0 - \vy^*}^2 \right] + C_0 (1 - \rho)^{T} \\
        & \le (1 - 2 \rho)^{T-1} \frac{\norm{\mA}}{\sqrt{\mu_x \mu_y}} C_0 + C_0 (1 - \rho)^{T} \\
        & \le  C_0 (1 - \rho)^T \left( \frac{\norm{\mA}}{\sqrt{\mu_x \mu_y}}  + 1 \right).
    \end{align*}
    where the last inequality is according to $(1 - 2 \rho)^{T-1} \le (1 - \rho)^{T}$ for $T \ge 2$.
\end{proof}
For the inner loop, we have the following lemma.
\begin{lemma}\label{lem:aipfb:inner}
    Consider the same assumption and the same definitions of $\eps_k$, $\theta$, $\rho$, $C$ and $C_0$ in Theorem \ref{thm:aipfb-out}. 
    Denote $\kappa_x = L_x / \mu_x$, $\kappa_y = L_y / \mu_y$ and $\tilde{\kappa} = \norm{\mA} / \sqrt{\mu_x \mu_y}$.
    In order to find $\eps_k$-optimal points $\vx_k$ of $g_k$, we need to run $\mathrm{AGD}$ $K_1$ steps, where
    \begin{align*}
        K_1 = \floor{\sqrt{ \frac{\kappa_y + \tilde{\kappa} }{1 + \tilde{\kappa} }}
        \log\left( \frac{ 320 C (1 - \rho) (\kappa_y + 2 \tilde{\kappa} + 1) }{ \rho (1 - \theta)  } \right)} + 1.
    \end{align*}
    And in order to obtain $\eps_k$-optimal point $\vy_k$ of $h_k$, we need to run $\mathrm{AGD}$ $K_2$ steps, where 
    \begin{align*}
        K_2 = \floor{\sqrt{ \frac{\kappa_x + \tilde{\kappa} }{1 + \tilde{\kappa} }}
        \log\left( \frac{80 C (\kappa_x  + 2 \tilde{\kappa} + 1) }{ \rho (1 - \theta) } \right)} + 1 .
    \end{align*}
\end{lemma}
\begin{proof}
    By Lemma \ref{lem:prox:Lipschitz} and Cauchy-Schwarz inequality, we have
    \begin{align*}
        \norm{\vy_k^* - \vy^*}^2
        &\le \norm{\vw_k - \vw^*}^2 \\
        & \le 2 \left( \norm{\vy_{k-1} - \vy^*}^2 + \sigma^2 \norm{\mA}^2 \norm{\tilde{\vx}_{k-1} - \vx^*}^2 \right) \\
        & = 2 \left( \norm{\vy_{k-1} - \vy^*}^2 +  2 (1 + \theta)^2 \frac{\mu_x}{\mu_y} \norm{\vx_{k-1} - \vx^*}^2 + 2 \theta^2 \frac{\mu_x}{\mu_y} \norm{\vx_{k-2} - \vx^*}^2 \right) \\
        & \le 2 \norm{\vy_{k-1} - \vy^*}^2 + \frac{16 \mu_x}{\mu_y} \norm{\vx_{k-1} - \vx^*}^2 + \frac{4 \mu_x}{\mu_y} \norm{\vx_{k-2} - \vx^*}^2.
    \end{align*}
    It follows that
    \begin{align*}
        \norm{\vy_{k-1} - \vy_k^*}^2
        & \le 2 \norm{\vy_{k-1} - \vy^*}^2 + 2 \norm{\vy_k^* - \vy^*}^2 \\
        & \le 6 \norm{\vy_{k-1} - \vy^*}^2 + \frac{32 \mu_x}{\mu_y} \norm{\vx_{k-1} - \vx^*}^2 + \frac{8 \mu_x}{\mu_y} \norm{\vx_{k-2} - \vx^*}^2. 
    \end{align*}
    By Theorem \ref{thm:aipfb-out}, we have $ \norm{\vy_{k-1} - \vy_k^*}^2 \le 40 C C_0 (1 - \rho)^{k-2} / \mu_y $.
    Denote $\tilde{\kappa} = \frac{\norm{\mA}}{\sqrt{\mu_x \mu_y}}$.
    Then the condition number of function $h_k$ is 
    \begin{align*}
        \frac{L_y + 1/\sigma}{\mu_y + 1/\sigma} = \frac{\kappa_y + \tilde{\kappa} }{1 + \tilde{\kappa} }.
    \end{align*}
    Thus, by Theorem \ref{thm:agd}, the first subproblem in step $k$ will need to run AGD with initial point $\vy_{k-1}$ at most $K_1$ steps where $K_1$ satisfies 
    \begin{align*}
        K_1 
        & = \floor{\sqrt{ \frac{\kappa_y + \tilde{\kappa} }{1 + \tilde{\kappa}}}
        \log\left( \frac{ 20 C C_0 (1 - \rho)^{k-2} (L_y + \mu_y + 2 / \sigma) }{\mu_y \eps_k} \right)} + 1 \\
        & = \floor{\sqrt{ \frac{\kappa_y + \tilde{\kappa} }{1 + \tilde{\kappa} }}
        \log\left( \frac{ 320 C (1 - \rho) (\kappa_y + 2 \tilde{\kappa} + 1) }{ \rho (1 - \theta)  } \right)} + 1.
    \end{align*}
    On the other hand, Lemma \ref{lem:prox:Lipschitz} and Cauchy-Schwarz inequality also imply
    \begin{align*}
        \norm{\vx_k^* - \vx^*}^2
        &\le \norm{\vz_k - \vz^*}^2 \\
        & \le 2 \left( \norm{\vx_{k-1} - \vx^*}^2 + \gamma^2 \norm{\mA}^2 \norm{\vy_k - \vy^*}^2 \right) \\
        & = 2 \left( \norm{\vx_{k-1} - \vx^*}^2 + \frac{\mu_y}{\mu_x} \norm{\vy_k - \vy^*}^2 \right).
    \end{align*}
    It follows that
    \begin{align*}
        \norm{\vx_{k-1} - \vx_k^*}^2
        \le 2 \norm{\vx_{k-1} - \vx^*}^2 + 2 \norm{\vx_k^* - \vx^*}^2
        \le 6 \norm{\vx_{k-1} - \vx^*}^2 +  \frac{4 \mu_y}{\mu_x} \norm{\vy_k - \vy^*}^2
        \le \frac{10}{\mu_x} C C_0 (1 - \rho)^{k-1}.
    \end{align*}
    The condition number of function $g_k$ is 
    \begin{align*}
        \frac{L_x + 1/\gamma}{\mu_x + 1/\gamma} = \frac{\kappa_x + \tilde{\kappa} }{1 + \tilde{\kappa} }.
    \end{align*}
    Thus, by Theorem \ref{thm:agd}, the second subproblem in step $k$ will need to run AGD with initial point $\vx_{k-1}$ at most $K_2$ steps where $K_2$ satisfies 
    \begin{align*}
        K_2 
        & = \floor{\sqrt{ \frac{\kappa_x + \tilde{\kappa} }{1 +\tilde{\kappa} }}
        \log\left( \frac{5 C C_0 (1 - \rho)^{k-1} (L_x + \mu_x + 2 / \gamma) }{\mu_x \eps_k} \right)} + 1 \\
        & = \floor{\sqrt{ \frac{\kappa_x + \tilde{\kappa} }{1 + \tilde{\kappa} }}
        \log\left( \frac{80 C (\kappa_x  + 2 \tilde{\kappa} + 1) }{ \rho (1 - \theta) } \right)} + 1.
    \end{align*}
\end{proof}

Then we can provide the proof of Theorem \ref{thm:aipfb}.
\begin{proof}[Proof of Theorem \ref{thm:aipfb}]
    Denote $\tilde{\kappa} = \frac{\norm{\mA}}{ \sqrt{\mu_x \mu_y}}$. \\
    By Lemma \ref{thm:aipfb-out}, in order to find $\eps$-saddle point of $f$, we need to run AIPFB  
    \begin{align*}
        K = \floor{ ( 4\tilde{\kappa} + 2) \log\left( \frac{C C_0}{\eps} \right)} + 1
    \end{align*}
    steps.
    Therefore, the number of total queries to Oracle (\ref{eq:oracle}) is upper bounded by
    \begin{align*}
        K( K_1 + K_2) 
        &= \tilde{\gO} \left( (1 + \tilde{\kappa}) \sqrt{ \frac{ \kappa_x + \kappa_y + \tilde{\kappa}}{1 + \tilde{\kappa}} } \log \left( \frac{\mu_x \norm{\vx_{0} - \vx^*}^2 + \mu_y \norm{\vy_{0} - \vy^*}^2}{\eps} \right) \right) \\
        &= \tilde{\gO} \left( \left( \sqrt{\tilde\kappa^2 + \kappa_x + \kappa_y + \tilde\kappa(\kappa_x + \kappa_y) } \right) \log \left( \frac{\mu_x \norm{\vx_{0} - \vx^*}^2 + \mu_y \norm{\vy_{0} - \vy^*}^2}{\eps} \right) \right).
    \end{align*}
\end{proof}

\subsection{An Improved upper bound for AIPFB}
For $L_x = L_y$ and $\mu_x \le \mu_y$, the complexity in Theorem \ref{thm:aipfb} becomes
\begin{align*}
    \tilde\gO\left(\frac{\norm{\mA}}{\sqrt{\mu_x \mu_y}} + \sqrt{\frac{\norm{\mA} L}{\mu_x^{3/2} \mu_y^{1/2}}} + \sqrt{\kappa_x + \kappa_y}\right),
\end{align*}
where $L = \max\{\norm{\mA}, L_x\}$. In this section, we improve term $\sqrt{\frac{\norm{\mA} L}{\mu_x^{3/2} \mu_y^{1/2}}}$ to be $\sqrt{\frac{\norm{\mA} L}{\mu_x \mu_y}}$ by Catalyst framework.

Without loss of generality, we can assume that $L_x = L_y$. Otherwise, one can rescale the variables and take $\hat{f} = f( \sqrt[4]{L_y / L_x} \vx, \sqrt[4]{L_x/L_y} \vy)$. It is not hard to check that this rescaling will not change condition numbers $\kappa_x, \kappa_y$, $\mu_x \mu_y$, the coupling matrix $A$ and increase $\max \{ \norm{\mA}, L_x , L_y \}$. 

We first consider the special case where $\mu_x = \mu_y$. The total queries to Oracle (\ref{eq:oracle}) needed by Algorithm \ref{algo:aipfb} to produce $\eps$-saddle point is at most
\begin{align*}
    \tilde\gO\left(\left( \frac{\sqrt{\norm{\mA} L}}{\mu_y} + \sqrt{\kappa_y}\right)\log\left( \frac{\mu_x \norm{\vx_{0} - \vx^*}^2 + \mu_y \norm{\vy_{0} - \vy^*}^2}{\eps} \right)\right).
\end{align*}

\begin{algorithm}[tb]
    \caption{Catalyst-AIPFB }\label{algo:catalyst-aipfb}
    \begin{algorithmic}[1]
        \STATE \textbf{Input:} function $f$, initial point $(\vx_0, \vy_0)$, smoothness $L_x, L_y$, strongly convex module $\mu_x < \mu_y$, run-time $T$, accuracy sequence $\{\eps_k\}_{k\ge 1}$. \\[0.1cm]
        \STATE \textbf{Initialize:} $\beta = \frac{L_x(\mu_y - \mu_x)}{L_x - \mu_y}$, $q = \frac{\mu_x}{\mu_x + \beta}$, $\theta = \frac{1 - \sqrt{q}}{1 + \sqrt{q}}$ and $\tilde\vx_0 = \vx_0$. \\[0.1cm]
        \STATE \textbf{for} $k = 1, \cdots, K$ \textbf{do}\\[0.1cm]
        \STATE\quad Let $f_k(\vx, \vy) = f(\vx, \vy) + \frac{\beta}{2} \norm{\vx - \tilde\vx_k}^2$.\\[0.15cm]
        \STATE\quad Obtain $\eps_k$-saddle point $(\vx_{k}, \vy_{k})$ of $f_k(\vx, \vy)$ by applying AIPFB.
        \STATE\quad $\tilde\vx_{k} = \vx_{k} + \theta (\vx_{k} - \vx_{k-1})$.
        \STATE\textbf{end for} \\[0.1cm]
        \STATE \textbf{Output:} $\vx_T, \vy_T$.
    \end{algorithmic}
\end{algorithm}

For the general case, without loss of generality we assume $\mu_x \le \mu_y$. 
Similar to Catalyst-DIPPA, We can apply Catalyst framework to accelerate the Algorithm \ref{algo:aipfb}. The details are presented in Algorithm \ref{algo:catalyst-aipfb}. 
Again we remark that the function $f_k$ in each subproblem is balanced: the condition number corresponding to $\vy$ is $\kappa_y$ and the condition number related to $\vx$ is
\begin{align*}
    \frac{L_x + \beta}{\mu_x + \beta} = \frac{L_x(L_x - \mu_y) + L_x (\mu_y - \mu_x)}{\mu_x(L_x - \mu_y) + L_x (\mu_y - \mu_x)} = \kappa_y,
\end{align*}
where we have recalled that $L_x = L_y$. 

By results of Catalyst \cite{yang2020catalyst}, the number of total queries to Oracle (\ref{eq:oracle}) is upper bounded by
\begin{align*}
    &\quad \tilde\gO\left( \sqrt{\frac{\mu_x + \beta}{\mu_x}} \right) \tilde
    \gO\left( \sqrt{\frac{\norm{\mA} L}{\mu_y^2}} + \sqrt{\kappa_y} \right)  \\
    &=  \tilde{\gO} \left(  \sqrt{\frac{\norm{\mA} L}{\mu_x \mu_y} } + \sqrt{\kappa_x + \kappa_y} \right).
\end{align*}
    
A formal statement of theoretical guarantee for Catalyst-AIPFB is presented as follows. 
\begin{thm}\label{thm:catalyst-aipfb}
    Assume that $g(\vx)$ is $L_x$-smooth and $\mu_x$-strongly convex and $h(\vy)$ is $L_y$-smooth and $\mu_y$-strongly convex. The total queries to Oracle (\ref{eq:oracle}) needed by Algorithm \ref{algo:catalyst-aipfb} to produce $\eps$-saddle point of $f(\vx, \vy) = g(\vx) + \inner{\vx}{A \vy} - h(\vy)$ is at most
    \begin{align*}
        \tilde\gO\left(\left(\sqrt{\frac{\norm{A} L}{\mu_x \mu_y}} + \sqrt{\kappa_x + \kappa_y}\right)\log\left( \frac{1}{\eps} \right)\right),
    \end{align*}
    where $L = \max \left\{ \norm{A}, L_x, L_y \right\}$, $\kappa_x = L_x / \mu_x$, $\kappa_y = L_y / \mu_y$ and the notation $\tilde\gO$ have omitted some logarithmic factors depending on $L_x$, $L_y$, $\norm{A}$, $\mu_x$ and $\mu_y$. 
\end{thm}

\end{document}